\documentclass{article}

\usepackage{PRIMEarxiv}

\usepackage[utf8]{inputenc} 
\usepackage[T1]{fontenc}    
\usepackage{hyperref}       
\usepackage{url}            
\usepackage{booktabs}       
\usepackage{amsfonts}       
\usepackage{nicefrac}       
\usepackage{microtype}      
\usepackage{xcolor}         

\usepackage{cite}
\usepackage{amsmath,amsthm,mathtools,amssymb}
\usepackage{graphicx}
\usepackage{algorithm} 
\usepackage{algorithmic}
\usepackage[algo2e]{algorithm2e}
\usepackage{multirow}
\usepackage{caption, makecell}
\usepackage{threeparttable}
\usepackage{enumitem}
\usepackage{listings}
\usepackage{wrapfig}
\usepackage{pifont}

\newcommand{\etal}{\textit{et al}. }
\newcommand{\ie}{\textit{i}.\textit{e}., }

\newtheorem{theorem}{Theorem}
\newtheorem{lemma}[theorem]{Lemma}

\newtheorem{definition}{Definition}
\newtheorem{proposition}{Proposition}

\newcommand{\method}{CGProNet}
\newcommand{\cmethod}{C2GProNet}

\title{Spatiotemporal Forecasting Meets Efficiency: Causal Graph Process Neural Networks
}

\author{
  Aref Einizade\\
  Institut Polytechnique de Paris \\
  Télécom Paris, LTCI \\
  \texttt{aref.einizade@telecom-paris.fr} \\
   \And
  Fragkiskos D. Malliaros \\
  Université Paris-Saclay \\
  CentraleSupélec, Inria \\
  \texttt{fragkiskos.malliaros@centralesupelec.fr} \\
  \AND
  Jhony H. Giraldo \\
  Institut Polytechnique de Paris \\
  Télécom Paris, LTCI \\
  \texttt{jhony.giraldo@telecom-paris.fr} \\
}

\begin{document}
\maketitle

\begin{abstract}
Graph Neural Networks (GNNs) have advanced spatiotemporal forecasting by leveraging relational inductive biases among sensors (or any other measuring scheme) represented as nodes in a graph. However, current methods often rely on Recurrent Neural Networks (RNNs), leading to increased runtimes and memory use. Moreover, these methods typically operate within 1-hop neighborhoods, exacerbating the reduction of the receptive field. Causal Graph Processes (CGPs) offer an alternative, using graph filters instead of MLP layers to reduce parameters and minimize memory consumption. This paper introduces the Causal Graph Process Neural Network (\method), a non-linear model combining CGPs and GNNs for spatiotemporal forecasting. \method~employs higher-order graph filters, optimizing the model with fewer parameters, reducing memory usage, and improving runtime efficiency. We present a comprehensive theoretical and experimental stability analysis, highlighting key aspects of \method. Experiments on synthetic and real data demonstrate \method's superior efficiency, minimizing memory and time requirements while maintaining competitive forecasting performance.
\end{abstract}


\section{Introduction}
\label{Introduction}

Forecasting of time series has gained substantial attention in recent years, finding diverse applications such as traffic projection \cite{li2018diffusion}, air pollution prediction \cite{li2018diffusion}, and energy consumption estimation \cite{hummon2012sub}.
Forecasting problems often involve predicting the functional metrics of the current time sample based on a specified number of previous time step entities \cite{lim2021time, cini2023graph}.
This approach traces back to Vector Auto-Regressive (VAR) models and their evolution into memory-aware neural network extensions, such as Recurrent Neural Networks (RNNs).
Despite their historical success, VAR-based models ignore potential connections between different sensors.
Time series data, comprising temporal measurements, is commonly captured by sensor-based systems, presenting an opportunity to leverage graph machine learning techniques \cite{cini2023graph}.
Graph-based methodologies prove particularly beneficial in applications such as air quality prediction, emphasizing the importance of accounting for geographic distances between sensors to capture similar functionalities \cite{li2018diffusion}.


Graph Neural Networks (GNNs) have recently transformed spatiotemporal forecasting, improving system robustness and accuracy by incorporating relational inductive biases \cite{benidis2022deep, cini2023graph}
{\color{black}.
However, GNNs often perform only 1-hop computations, overlooking potential long-range interactions.
Moreover, the high parameter count in GNN-based methods can limit their application due to increased memory demands \cite{cini2023graph,cini2023scalable}.
This trend in spatiotemporal forecasting may result in resource-intensive models as observed in other fields in machine learning \cite{strubell2020energy}, limiting their accessibility and practicality, especially in scenarios requiring cost-effective real-time processing on edge devices.
A cost-effective processing system could be exemplified by a distributed forecasting system where each edge device has limited computational and power resources.
}
Leveraging Causal Graph Processes (CGPs) \cite{mei2016signal,ortega2018graph,leus2023graph}, we propose a lightweight spatiotemporal GNN model for forecasting, achieving a superior balance among runtime, memory usage, and accuracy, as illustrated in Figure \ref{MotivFig}.



\begin{wrapfigure}{r}{0.5\textwidth}
    \centering
    \vspace{-10pt}
    \includegraphics[width=0.39\columnwidth]{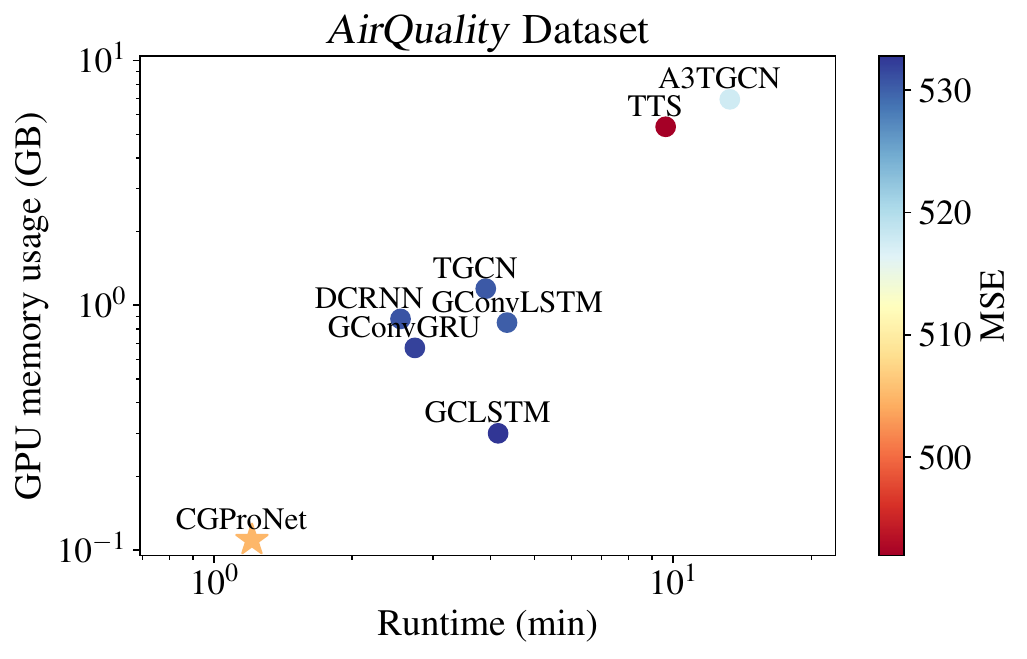}
    \caption{Comparison of Mean Squared Error (MSE), GPU memory consumption, and runtime in the forecasting task on the \textsl{AirQuality} dataset.}
    \label{MotivFig}
    \vspace{-10pt}
\end{wrapfigure}

In this paper, we present a novel Auto-Regressive (AR) process leveraging graph filters and non-linear aggregations.
Specifically, we introduce the Causal Graph Process Neural Network (\method), a drastic departure from the current state-of-the-art RNN-based methods.
Our model benefits from global (\ie beyond 1-hop interactions), local, and temporal relationships, and improves memory and computational complexity.
We also theoretically analyze the stability properties of \method, demonstrating its adept utilization of the sparsity inherent in the underlying process.
This is a notable advantage due to the intrinsic sparsity existing in real-world problems.
We test \method~in common benchmark datasets for spatiotemporal forecasting, where our model shows competitive results against previous methods.

This work makes the following main contributions:
\begin{itemize}[leftmargin=0.5cm]
    \item We introduce a general AR process that incorporates graph filters and non-linear aggregations.
    This represents a novel approach compared to previous methods.
    \item The proposed \method~model is designed to efficiently leverage global, local, and temporal relationships, leading to improvements in memory consumption and runtime.
    \item We conduct an extensive theoretical and experimental analysis of the stability properties of \method, showcasing its ability to leverage the sparsity inherent in the underlying process.
    \item The experiments demonstrate that \method~achieves an optimal balance between accuracy, runtime, and memory consumption across a wide range of datasets.
\end{itemize}


\section{Related Work}
\label{sec:related_work}


In the following, we briefly review the timeline of related spatiotemporal forecasting approaches.
For a deep and thorough survey, please refer to the review papers \cite{benidis2022deep,cini2023graph}.

\textbf{Signal processing-based methods}.
The classical causal modeling, rooted in the VAR scheme \cite{bolstad2011causal}, suffers from a substantial increase in optimization variables due to unconstrained VAR coefficient matrices.
Structural VAR (SVAR) models, introduced to extend recursive principles to spatiotemporal time series, mitigate this issue by enforcing shared sparsity patterns among VAR matrices \cite{bolstad2011causal}.
With the emergence of Graph Signal Processing (GSP) \cite{ortega2018graph,leus2023graph}, the underlying descriptive structures were unified through {\color{black} the lens of graph polynomial filtering.
For instance, GP-VAR \cite{isufi2019forecasting,isufi2016autoregressive} and GP-VARMA \cite{isufi2016autoregressive} were proposed to adapt the classic VAR and VARMA processes to graphs.
The only GSP-based graph polynomial method for considering directed graphs is the CGP model \cite{mei2016signal}, which, contrary to GP-VAR and GP-VARMA, enjoys a spatial filtering interpretation.}
The CGP model treats each temporal sample as a (filtered) graph signal on the underlying graph, offering an immediate advantage of significantly reduced parameters.
These signal processing models are currently limited to \textit{linear} cases, constraining their applicability in real-world scenarios involving non-linearities.

\textbf{GNN-based techniques}.
Depending on the merging approach of learned representations for spatial and temporal domains, the GNN models can be roughly categorized into Time-and-Space (T\&S), Time-Then-Space (TTS), and Space-Then-Time (STT) methods \cite{cini2023graph}.

In T\&S, integrating time and space modules is crucial.
A seminal contribution to modern T\&S frameworks comes from Seo \etal \cite{seo2018structured}, who introduced GConvGRU and GConvLSTM networks, merging Gated Recurrent Unit (GRU) and Long Short-Term Memory (LSTM) modules with GNNs, thus extending traditional RNNs.
Li \etal \cite{li2018diffusion} proposed the Diffusion Convolutional Recurrent Neural Network (DCRNN), employing an encoder-decoder architecture with scheduled sampling to capture both temporal dependencies and graph topology for traffic flow forecasting.
Alternatively, Chen \etal \cite{chen2022gc} introduced GCLSTM, an end-to-end network utilizing LSTM modules for temporal dependencies, and GNNs for learning from node features derived from underlying graph connections.
Notably, the high computational complexity is a significant drawback in employing T\&S frameworks.


TTS models adopt a sequential approach, initially processing time step-specific information and subsequently employing GNN operations to leverage message passing between spatially connected nodes.
Cini \etal \cite{cini2023graph} followed the TTS paradigm to develop a Spatiotemporal GNN (STGNN).
They implemented a shared GRU module among nodes to capture temporal dependencies and provide distinct learned features for each node.
These features were then propagated through the underlying graph using a DCRNN. 
They also added Multi-Layer Perceptron (MLP) layers for encoding-decoding the underlying embeddings.
It is crucial to note that the division of TTS networks into temporal and spatial segments can introduce propagation information bottleneck effects during training.

In designing STT architectures, the key concept is to enhance node representations initially and then aggregate temporal information.
Zhao \etal \cite{zhao2019t} introduced a Temporal Graph Convolutional Network (TGCN) model, merging GNNs and GRUs to learn spatial and temporal connections simultaneously, notably for traffic forecasting.
Extending TGCNs, Bai \etal \cite{bai2021a3t} proposed an Attention Temporal Graph Convolutional Network (A3T-GCN), emphasizing adjacent time step importance through GRU modules, GNNs, and attention mechanisms.
This allows simultaneous embedding of global temporal tendencies and spatial-graph connections in traffic flow data streams.
However, the frequent use of RNNs in aggregation stages makes STT models prone to high computational complexity.

In summary, designing GNN-based techniques for spatiotemporal forecasting faces two primary challenges: i) RNN-based temporal encoders introduce high memory and computational complexity, and ii) conventional GNN operations \cite{kipf2017semi} limit models to exploiting only 1-hop information, overlooking potential long-range node interactions in the graph.
\method~addresses both challenges by utilizing CGPs, yielding a memory and computationally efficient GNN framework.


\section{Proposed Framework}
\label{sec:method}

\noindent \textbf{Problem definition.}
Let \(\mathcal{G}=(\mathcal{V},\mathcal{E},\mathbf{A})\) be a directed graph with $\mathcal{V}=\{\mathop{v}_1,\ldots,\mathop{v}_N\}$ the set of nodes, ${\mathcal{E}\subseteq \{(\mathop{v}_i,\mathop{v}_j)\mid \mathop{v}_i,\mathop{v}_j\in \mathcal{V}\;{\textrm {and}}\;\mathop{v}_i\neq \mathop{v}_j\}}$ is the set of edges between nodes $\mathop{v}_i$ and $\mathop{v}_j$, and $\mathbf{A}\in\mathbb{R}^{N\times N}$ the adjacency matrix of the graph. 
A graph signal is a function $x:\mathcal{V} \to \mathbb{R}$, represented as $\mathbf{x} = [x_1,\ldots,x_N] \in \mathbb{R}^N$, where $x_i$ is the graph signal evaluated on the $i$th node \cite{ortega2018graph,leus2023graph}. 
Let $\mathbf{X}=[ \mathbf{x}_0,\mathbf{x}_1,\ldots,\mathbf{x}_{K-1}] \in\mathbb{R}^{N\times K}$ be a matrix of observations that contains \(K\)-length temporal signals living on each node of $\mathcal{G}$, where $\mathbf{x}_i \in \mathbb{R}^N;~\forall~0 \leq i < K$.
Given a temporal window $\mathbf{X}_M = [\mathbf{x}_{k-M}, \dots, \mathbf{x}_{k-1}] \in\mathbb{R}^{N \times M}$ with $M < K$, the objective of this work is to train a model $f(\mathbf{A},\mathbf{X}_M;\{\boldsymbol{\theta}_i\}_{i=1}^M)$ with parameters $\{\boldsymbol{\theta}_i\}_{i=1}^M$ to predict the future time-step $\mathbf{x}_k~\forall~k \geq M$.
The observation matrix \(\mathbf{X}\) can be thought of as \(K\) spatial graph signals on the graph $\mathcal{G}$.

\subsection{General Formulation}

We can filter a graph signal using an \(L\)-order graph polynomial filter as \cite{mei2016signal}:
\begin{equation}
\label{G_Ac}
P(\mathbf{A}, \mathbf{c})=\sum_{i=0}^{L}{c_i\mathbf{A}^i}=c_0\mathbf{I}+c_1\mathbf{A}+\ldots+c_L\mathbf{A}^L,
\end{equation}
where \(\mathbf{I}\in\mathbb{R}^{N\times N}\) is the identity matrix, and \(\mathbf{c}=[c_0,\ldots,c_L]^\top\) denotes a vector containing the scalar graph filter coefficients. 
Relying on the concept of polynomial graph filters, we assume $\mathbf{x}_k$ follows a general framework for the non-linear causal graph process as follows:
\begin{equation}
\label{General}
\mathbf{x}_{k} = \texttt{AGG}\left(\{\texttt{GF}_{\mathcal{G},\boldsymbol{\phi}_i}\mathbf{x}_{k-i}\}_{i=1}^M \right)+\mathbf{w}_{k},
\end{equation}
where, for any \(i=1,\hdots,M\),  $\texttt{GF}_{\mathcal{G},\boldsymbol{\phi}_i}\in\mathbb{R}^{N\times N}$ can have any form of polynomial graph filters (with parameters \(\boldsymbol{\phi}_i\)), \(\mathbf{w}_{k}\) is an instantaneous exogenous noise, and $\texttt{AGG}(\cdot)$ denotes an aggregation function. 
The term \textit{causal} stems from i) the dependency of the current time step on the $M$ previous ones, and ii) the fact that the underlying graph is directed.
The directionality of the edges describes causal effects between the graph nodes.
Therefore, causality in both spatial and temporal dimensions is considered.
The aggregation function takes the form:
\begin{equation}
\label{eqn:aggregation_function}
\texttt{AGG}\left(\{\mathbf{x}_{i}\}_{i=1}^M\right)=\sum_{i=1}^M{\alpha_i  \sigma(\mathbf{x}_{i})},
\end{equation}
where $\alpha_i$ is a parameter and \(\sigma(\cdot)\) is a non-linear function.
Please notice that our framework in \eqref{General} is a generalization of the well-known VAR \cite{bolstad2011causal} and CGP \cite{mei2016signal} models.
More precisely, our framework in \eqref{General} reduces to the VAR process when: i) $\{\alpha_i=1\}_{i=1}^M$, ii) $\sigma(\cdot)$ is a linear function, and iii) $\texttt{GF}_{\mathcal{G},\boldsymbol{\phi}_i}=\mathbf{R}_i\in\mathbb{R}^{N\times N}$, where \(\{\mathbf{R}_i\}_{i=1}^M\) are unconstrained coefficient matrices.
Similarly, we can recover the classical CGP when $\texttt{GF}_{\mathcal{G},\boldsymbol{\phi}_i}=\sum_{j=0}^{i}{\phi_{ij}\mathbf{A}^j}$ and the conditions i) and ii) for being VAR also hold.
For more details, please see Appendix \ref{app:VAR_CGP}.

\begin{figure}
    \centering
    \includegraphics[width=0.98\textwidth]{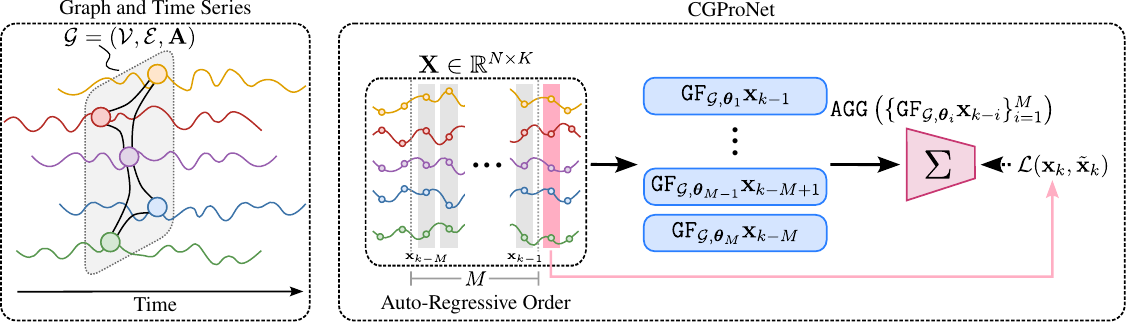}
    \caption{Pipeline of \method. 
    We process subsequent \(M\) spatial time steps ($\mathbf{x}_{k-1},\mathbf{x}_{k-2},\dots,\mathbf{x}_{k-M}$) with polynomial graph filters \(\{\texttt{GF}_{\mathcal{G},\boldsymbol{\theta}_i}\}_{i=1}^M\) to forecast $\mathbf{x}_{k}$ after the aggregation function $\texttt{AGG}\left(\{\texttt{GF}_{\mathcal{G},\boldsymbol{\theta}_i}\mathbf{x}_{k-i}\}_{i=1}^M \right)$.
    Finally, we optimize the difference between the predicted $\tilde{\mathbf{x}}_k$ and ground-truth time step $\mathbf{x}_k$ to train our \method~model.}
    \label{fig:pipeline}
\end{figure}

Following our general framework in \eqref{General}, we propose to forecast $\mathbf{x}_k$ using the following GNN model:
\begin{equation}
    \label{eqn:CGProNet_GNN}
    \tilde{\mathbf{x}}_k = \texttt{AGG}\left(\{\texttt{GF}_{\mathcal{G},\boldsymbol{\theta}_i}\mathbf{x}_{k-i}\}_{i=1}^M \right),
\end{equation}
where $\texttt{GF}_{\mathcal{G},\boldsymbol{\theta}_i}$ is a polynomial graph filter with learnable parameters $\boldsymbol{\theta}_i$, and $\texttt{AGG}(\cdot)$ takes the form as in \eqref{eqn:aggregation_function}.
In this paper, we propose a discrete polynomial graph filter in Section \ref{sec:discrete_polynomial_filter}.
Similarly, we use $\tanh(\cdot)$ as the non-linear function $\sigma(\cdot)$ in \eqref{eqn:aggregation_function}.
Finally, we can train \method~with some loss function $\mathcal{L}(\mathbf{x}_{k},\tilde{\mathbf{x}}_k)$ typically used in the context of spatiotemporal forecasting like Mean Squared Error (MSE) or Mean Absolute Error (MAE).
We illustrate our general framework in Figure \ref{fig:pipeline}.

\subsection{Causal Graph Process Neural Network (\method)}
\label{sec:discrete_polynomial_filter}

The proposed \method~model uses discrete polynomial graph filters as follows:
\begin{equation}
\label{Def_GPi}
\texttt{GF}_{\mathcal{G},\boldsymbol{\theta}_i}=P(\mathbf{A}, \boldsymbol{\theta}_i)=\sum_{j=0}^{i}{\theta_{ij}\mathbf{A}^j},
\end{equation}
where \(\boldsymbol{\theta}_i=[\theta_{i0}, \theta_{i1}, \hdots,\theta_{ii}]^\top \in \mathbb{R}^{i+1}\).
Therefore, the full GNN model following \eqref{eqn:CGProNet_GNN} is given as:
\begin{equation}
    \label{eqn:CGProNet}
    \tilde{\mathbf{x}}_k = \sum_{i=1}^M{\alpha_i  \sigma\left( \sum_{j=0}^{i}{\theta_{ij}\mathbf{A}^j} \mathbf{x}_{k-i}\right)}.
\end{equation}
The graph signal \(\mathbf{x}_{k-i}\) is the shifted version of the current graph signal \(\mathbf{x}_{k}\) by \(i\) in the time domain, and $P(\mathbf{A}, \boldsymbol{\theta}_i) \mathbf{x}_{k-i}$ is shifted to the \(i\)-hop distance in the spatial graph domain.
Therefore, we explicitly rely on higher-order graph filters (beyond 1-hop) for $M>1$.
The main intuition for this kind of modeling is transferring activity on the network at some fixed speed, \ie one spatial graph shift per time step.
Therefore, the information of the current time step cannot be affected by network order higher than that fixed speed.
This interpretation can also be considered as an extension of the spatial dimension of the light cone \cite{goerg2012licors,goerg2013mixed}, due to the possibility of living on an irregular manifold rather than a regular (uniformly sampled) lattice space \cite{mei2016signal}.

\subsection{Stability Analysis of \method}
\label{sec:stability_analysis}

We theoretically analyze the stability \cite{gama2020stability,wang2022graph} and prediction performance of the \method~model under the perturbation of the underlying graph and network parameters.
Particularly, we consider the deviation of the adjacency matrix $\mathbf{A}$ with the perturbation matrix $\mathbf{E}$ as follows:
\begin{equation}
\label{pert_model}
\hat{\mathbf{A}} = \mathbf{A} + \mathbf{E},
\end{equation}
where the adjacency matrix powers and error matrix \(\mathbf{E}=\hat{\mathbf{A}}-\mathbf{A}\) are upper-bounded by $\max_{1\le i\le M}(\|\mathbf{A}^i\|_2)\le L;\:\:\:\|\mathbf{A}-\hat{\mathbf{A}}\|_2\le\delta_\mathbf{A}$.
Similarly, for the graph filter coefficients \(\boldsymbol{\theta}\) and their associated deviations \(\hat{\boldsymbol{\theta}}\), we have $\|\boldsymbol{\theta}\|_1\le\rho_{\boldsymbol{\theta}}$, and $\|\boldsymbol{\theta}-\hat{\boldsymbol{\theta}}\|_1\le \delta_{\boldsymbol{\theta}}$.
\begin{definition}
\label{Lipschitz}
We call a function $\sigma(\cdot)$ as Lipschitz if there exists a positive constant $\upsilon$ such that $\forall~x_1, x_2 \in \mathbb{R}: \vert\sigma(x_1)-\sigma(x_2)\vert \le \upsilon \vert x_1-x_2\vert$.
\end{definition}
In the following, we first analyze the stability of the graph polynomial filter \(P(\mathbf{A},\boldsymbol{\theta}_i)\) under the inaccuracies for the adjacency matrix \(\mathbf{A}\) and graph filter coefficients \(\boldsymbol{\theta}_i\). The next proposition gives the desired upper bound. All proofs are provided in Appendix \ref{app:proofs}.
\begin{proposition}
\label{prop1}
Consider the graph filter deviations as \(\hat{\theta}_{ij}=\theta_{ij}+z_{ij}\). Then, with the assumptions of the adjacency matrix powers \(\{\mathbf{A}^i\}_{i=1}^M\), error matrix \(\mathbf{E}=\hat{\mathbf{A}}-\mathbf{A}\), graph filter coefficients \(\boldsymbol{\theta}_i\) and their associated deviations \(\hat{\boldsymbol{\theta}}_i\) are respectively upper-bounded by \(\max_{1\le i\le M}(\|\mathbf{A}^i\|_2)\le L\), \(\|\mathbf{A}-\hat{\mathbf{A}}\|_2\le\delta_\mathbf{A}\), \(\|\boldsymbol{\theta}_i\|_1\le\rho_{\boldsymbol{\theta}_i}\), and \(\|\boldsymbol{\theta}_i-\hat{\boldsymbol{\theta}}_i\|_1=\|\mathbf{z}_i\|_1\le \delta_{\boldsymbol{\theta}_i}\), an upper bound for the perturbed graph filter polynomial \(P(\hat{\mathbf{A}},\hat{\boldsymbol{\theta}}_i)\) can be stated as follows:
\begin{equation}
\label{GF_bound}
\begin{split}
&\left\|P(\hat{\mathbf{A}},\hat{\boldsymbol{\theta}}_i)\right\|_2\le (\rho_{\boldsymbol{\theta}_i}+\delta_{\boldsymbol{\theta}_i})(L+\delta_{\mathbf{A}}\hat{L}_M(\delta_\mathbf{A})),
\end{split}
\end{equation}
where $\hat{L}_M(\delta)=\max_{1\le i\le M}{\frac{(L+\delta)^i-L^i}{\delta}}$.

\end{proposition}

The bound in \eqref{GF_bound} implies that the polynomial graph filter \eqref{Def_GPi} is affected by i) the process sparsity \(\rho_{\boldsymbol{\theta}_i}\), ii) the scale and accuracy of the input graph, and iii) the AR order $M$.
The next proposition studies the stability of the graph polynomial filters \eqref{Def_GPi}.
\begin{proposition}
\label{prop2}
Consider the graph filter deviations as \(\hat{\theta}_{ij}=\theta_{ij}+z_{ij}\). Then, holding the assumptions of Proposition \ref{prop1}, the stability of the graph filter polynomial \(P(\mathbf{A},\boldsymbol{\theta}_i)\) can be stated as follows:
\begin{equation}
\label{GF_stab}
\left\|P(\hat{\mathbf{A}},\hat{\boldsymbol{\theta}}_i)-P(\mathbf{A},\boldsymbol{\theta}_i)\right\|_2 \le \rho_{\boldsymbol{\theta}_i}\delta_{\mathbf{A}}\hat{L}_M(\delta_\mathbf{A}) + \delta_{\boldsymbol{\theta}_i}(L+\delta_{\mathbf{A}}\hat{L}_M(\delta_\mathbf{A})).
\end{equation}
\end{proposition}

The first term in the bound \eqref{GF_stab} is heavily affected by the process sparsity $\rho_{\boldsymbol{\theta}_i}$ and the accuracy of the underlying graph $\delta_{\mathbf{A}}$.
Furthermore, the second term in \eqref{GF_stab} illustrates the effect of scale $L$ and coefficient accuracy $\delta_{\boldsymbol{\theta}_i}$. 
The next theorem investigates the conditions under which the proposed \method~model can benefit from these stability properties.
\begin{theorem}
\label{thm1}
Under holding the assumptions of Propositions \ref{prop1} and \ref{prop2}, with the assumptions of the non-linearity function $\sigma(\cdot)$ being Lipschitz, \(\sigma(0)=0\), mixing coeffients \(\boldsymbol{\alpha}\) and their associated deviations \(\hat{\boldsymbol{\alpha}}=\boldsymbol{\alpha}+\mathbf{e}\) are respectively upper-bounded as \(\|\boldsymbol{\alpha}\|_1\le\rho_{\boldsymbol{\alpha}}\), and \(\|\boldsymbol{\alpha}-\hat{\boldsymbol{\alpha}}\|_1\le\delta_{\boldsymbol{\alpha}}\), the difference between the true \(\mathbf{x}_{k}\) and predicted outputs \(\tilde{\mathbf{x}}_k\) of the proposed \method~for the future time step \(k\) is upper-bounded as:
\begin{equation}
\label{stab_bound}
\resizebox{\hsize}{!}{$\|\tilde{\mathbf{x}}_k-\mathbf{x}_{k}\|_2 \le\rho_{\boldsymbol{\alpha}}\left(\rho_{\boldsymbol{\theta}}\delta_{\mathbf{A}}\hat{L}_M(\delta_\mathbf{A}) + \delta_{\boldsymbol{\theta}}(L+\delta_{\mathbf{A}}\hat{L}_M(\delta_\mathbf{A}))\right) \|\mathbf{X}\|_{1,2} +\delta_{\boldsymbol{\alpha}}(\rho_{\boldsymbol{\theta}}+\delta_{\boldsymbol{\theta}})(L+\delta_{\mathbf{A}}\hat{L}_M(\delta_\mathbf{A})) \|\mathbf{X}\|_{1,2}.$}
\end{equation}
\end{theorem}

The bound in (\ref{stab_bound}) reveals interesting aspects of \method's stability against possible perturbations.
Precisely, Theorem \ref{thm1} states that \method~can also benefit from the process and mixing sparsity $\rho_{\boldsymbol{\theta}}$ and $\rho_{\boldsymbol{\alpha}}$, respectively.
This suggests that using \(\ell_1\)-norm regularization in the loss function of \method~could enhance its stability.
Similarly, the bound in \eqref{stab_bound} relies on \(L\) and \(\hat{L}_M(\delta_\mathbf{A})\) (upper bounded by the adjacency matrix powers), which implies that a high AR order could make the network more susceptible to perturbations.
The bound on the adjacency matrix powers \(L\) can also have direct effects on the stability of the network, which motivates the use of normalization techniques on the adjacency matrix.
Intuitively, the higher the error bounds on the adjacency matrix \(\delta_{\mathbf{A}}\) and filter coefficients \(\delta_{\boldsymbol{\theta}}\), the more vulnerable the network in terms of stability.
The capability of the proposed network for benefiting from the sparsity of the underlying process obtained from (\ref{stab_bound}) is experimentally validated in Section \ref{sec:experiments_results}.

\subsection{Computational and Memory Complexity of \method}

The number of learnable parameters of \method~is given by \(M+\frac{M(M+3)}{2}\), corresponding to \(\{\alpha_i\}_{i=1}^M\) and \(\{\boldsymbol{\theta}_i\}_{i=1}^M\).
Since \(M\ll N\), \method~can handle intricated non-linear relationships through the time steps while maintaining small memory usage.
As the learnable parameters do not depend on the graph size $N$ or the number of temporal samples $K$, \method~enjoys a serious reduction of the number of learnable parameters, which is comprehensively investigated in more detail in the experiments.

\method~has the dominant computational complexity of \(\mathcal{O}(\frac{M(M+3)}{2}|\mathcal{E}|)\approx\mathcal{O}(M^2\vert \mathcal{E} \vert)\) in the case of leveraging the recursive diffusion implementation for the graph filters \cite{defferrard2016convolutional}.
The computational complexity is linear in terms of the number of edges and quadratic for AR order, which makes it desirable for learning from large sparse graphs and not very large AR orders.
The setup of large-scale sparse graphs and low AR orders is a typical scenario in real-world data.
Alternatively, \method~could also be implemented by precomputing the matrix powers \(\{\mathbf{A}^i\}_{i=1}^M\) to avoid the recursion of the diffusion strategy. 
However, the precomputing strategy is only feasible for medium-size graphs since $\mathbf{A}^i$ becomes dense quickly when $i>1$.
{\color{black}
The number of parameters in \method~could be reduced by adopting \textit{continuous} polynomial graph filters.
Appendix \ref{C2GProNet_det} provides the continuous extension of \method~along with its theoretical stability analysis and experimental results.
}

\subsection{{\color{black}{\method~to Forecast Multiple Horizons}}}

{
\color{black}

We extend \method~for the case of forecasting multiple horizons $H$ by slightly adjusting \eqref{eqn:CGProNet}.
We outline two approaches: (i) MLP head and (ii) adaptive prediction with associative weights.

\textbf{MLP head.}
In the MLP-head approach for multiple horizons forecasting, called \method$_{\text{MLP}}$, we transform the dimension of the output of \method~using an MLP with parameters $\boldsymbol{\Phi}_H\in\mathbb{R}^{1\times H}$ and activation function $\sigma(\cdot)$.
The MLP head adapts the output dimension of the regular \method~to the number of horizons $H$.
Therefore, the output is now a matrix $\tilde{\mathbf{X}}_k\in\mathbb{R}^{N\times H}$,  which is formulated as:
\begin{equation}
    \label{eqn:CGProNet2}
    \tilde{\mathbf{X}}_k = \sigma\left(\left[\sum_{i=1}^M{\alpha_i  \sigma\left( \sum_{j=0}^{i}{\theta_{ij}\mathbf{A}^j} \mathbf{x}_{k-i}\right)}\right]\boldsymbol{\Phi}_H\right).
\end{equation}
The main advantage of \eqref{eqn:CGProNet2} is the simplicity in both complexity and number of learnable parameters since it only adds $H$ learnable parameters to the base \method.

\textbf{Adaptive prediction with associative weights.}
We can incorporate temporal dependency in multiple horizons forecasting using the following definition:
\begin{equation}
    \label{eqn:recursion}
    f_{M}\left(\mathbf{C};\mathbf{A},\{\boldsymbol{\theta}_i\}_{i=1}^{M},\{\alpha_i\}_{i=1}^{M}\right):=\sum_{i=1}^M{\alpha_i  \sigma\left( \sum_{j=0}^{i}{\theta_{ij}\mathbf{A}^j} \mathbf{c}_{M-i+1}\right)},
\end{equation}
where $\mathbf{C}=[\mathbf{c}_1,\hdots,\mathbf{c}_M]\in\mathbb{R}^{N\times M}$.
Using \eqref{eqn:recursion}, we can express the \(h\)-th predicted time sample (for $h=1,\hdots,H$) as:
\begin{equation}
\label{eqn:CGProNet_Adaptive}
\tilde{\mathbf{x}}_{k+h-1}=f_{M}\left(\mathbf{X}_k^{(h)};\mathbf{A},\{\boldsymbol{\theta}^{(h)}_i\}_{i=1}^{M},\{\alpha_i\}_{i=1}^{M}\right),
\end{equation}
where we have learnable parameters $\boldsymbol{\theta}^{(h)}_i$ for $h=\{1,\dots,H \}$, $\mathbf{X}_k^{(1)}=[\mathbf{x}_{k-M},\hdots,\mathbf{x}_{k-1}]\in\mathbb{R}^{N\times M}$, and $\mathbf{X}_k^{(h>1)}=[\mathbf{x}_{k-M+h-1},\hdots,\mathbf{x}_{k-1},\tilde{\mathbf{x}}_{k},\hdots,\tilde{\mathbf{x}}_{k+h-2}]\in\mathbb{R}^{N\times M}$.
In other words, the outputs of the model in \eqref{eqn:CGProNet_Adaptive}, called \method$_\text{Ada}$, for $h-1$ previous time samples are included in the prediction process of the $h$-th time sample.
The number of parameters of \method$_\text{Ada}$ scales linearly with \(H\).
Appendix \ref{app:additional_multiH_forecasting} presents one additional extension for multiple horizon forecasting with some results.







}

\section{Experimental Results}
\label{sec:experiments_results}

In this section, we study different aspects of \method~on synthetic and real-world time series regarding performance and stability properties.
Similarly, we conduct an ablation study to analyze different aspects of the design choices of \method.
We compare \method~against the state-of-the-art methods TTS \cite{cini2023graph}, DCRNN \cite{li2018diffusion}, A3TGCN \cite{bai2021a3t}, GCLSTM \cite{chen2022gc}, GConvGRU \cite{seo2018structured}, GConvLSTM \cite{seo2018structured}, TGCN \cite{zhao2019t}, \textcolor{black}{GGNM \cite{xiong2024gated}, and STCN \cite{liu2022spatial,gao2023dynamic}.}


\textbf{Implementation details.}
We use PyTorch Geometric Temporal \cite{rozemberczki2021pytorch} to implement previous state-of-the-art methods.
We use Torch Spatiotemporal (TSL) \cite{Cini_Torch_Spatiotemporal_2022} to work with the real-world time series.
We construct the underlying graphs using the thresholded kernelized distances between sensors \cite{Cini_Torch_Spatiotemporal_2022}.
We take \(M=3\) for all methods and datasets unless we mention it.
The detailed hyperparameters for the synthetic and real-world datasets can be found in Appendix \ref{syn_set} and \ref{real_set}, respectively.

\textbf{Synthetic spatiotemporal time series.}
Here, we generate directed binary Erd\H{o}s-R\'enyi graphs $\mathcal{G}$ with $N$ nodes.
Therefore, we use the generative model in (\ref{General}) with the polynomial graph filter (\ref{Def_GPi}) to produce spatiotemporal time series in $\mathcal{G}$.
We set $\alpha_i=1;~\forall~1 \leq i \leq M$ and $\sigma(\cdot)=\sigma(\cdot)$ in \eqref{eqn:aggregation_function}.
Finally, we model $\mathbf{w}(k)$ in \eqref{General} as Gaussian noise with zero mean and variance one.
Further details about the synthetic data are provided in Appendix \ref{app:synthetic_dataset}.

We perform experiments in the synthetic data with variations in i) the Signal-to-Noise Ratio (SNR), ii) the number of time steps $K$, iii) the number of nodes $N$ of $\mathcal{G}$, and iv) the AR order $M$.
We segment the synthetic data into training-validation-testing as $50\%$-$25\%$-$25\%$ of data portions, and the best model on the validation set is used to produce the results for final evaluation on the testing set.

Table \ref{Table1} shows the forecasting results in terms of relative Mean Squared Error (rMSE) of \method~against baseline methods.
\method~shows superior forecasting performance and running times over previous methods, highlighting the advantages of our model.
We also observe in Table \ref{Table1} that the forecasting performance of \method~is robust against different values of $K$, $N$, and $M$.
Besides, the running time scales well with an increasing number of nodes $N$, time steps $K$, and even the AR order $M$.
Finally, it is worth mentioning that TTS (considered the state-of-the-art method) fails to handle the case of limited data in almost any setting.

\begin{table*}[t]
\caption{The forecasting results in terms of rMSE on the synthetic spatiotemporal time series across different settings. The best-performing method on each case is shown in \textbf{bold}.}
\label{Table1}
\begin{center}
\resizebox{\textwidth}{!}{
\begin{tabular}{rcccccccccccccccr}
\toprule
 
\multirow{2}{*}{\textbf{Method}}  & \multicolumn{3}{c}{SNR} & \multicolumn{3}{c}{$K$} & \multicolumn{3}{c}{$N$} & \multicolumn{3}{c}{$M$} \\ 
\cmidrule(l){2-4} \cmidrule(l){5-7} \cmidrule(l){8-10} \cmidrule(l){11-13}

& $-10$ & $0$ & $10$ & $50$ & $100$ & $500$ & $100$ & $500$ & $1000$ & $3$ & $5$ & $7$  \\ \midrule

& \multicolumn{12}{c}{rMSE}\\ \midrule

DCRNN & $1.00$ & $0.740$ & $0.333$

& $0.796$ & $0.750$ &  $0.732$

& $0.990$ & $0.982$ & $0.986$

& $0.815$ & $0.814$ & $0.817$\\ 

TTS & $1.196$ & $0.957$ & $0.480$ &

$1.243$ & $0.923$ & $0.750$

& $1.205$ & $1.401$ & $1.130$ 

& $4.607$ & $2.461$ & $7.125$ \\ 

GCLSTM & $1.029$ & $0.748$ & $0.333$

& $0.809$ & $0.760$ & $0.733$ 

& $1.003$ & $0.984$ & $0.987$

& $0.816$ & $0.816$ & $0.823$\\ 






GConvGRU & $1.005$ & $0.740$ & $0.333$

& $0.797$ & $0.747$ & $0.732$

& $0.994$ & $0.982$ & $0.985$

& $0.815$ & $0.814$ & $0.816$\\ 

GConvLSTM & $1.034$ & $0.750$ &  $0.334$

& $0.822$ & $0.762$ & $0.734$ 

& $1.005$ & $0.984$ & $0.987$

& $0.816$ & $0.817$ & $0.822$\\ 






\method~(ours) & $\mathbf{0.955}$ & $\mathbf{0.704}$ & $\mathbf{0.302}$

& $\mathbf{0.717}$ & $\mathbf{0.710}$ & $\mathbf{0.709}$ 

& $\mathbf{0.951}$ & $\mathbf{0.955}$ & $\mathbf{0.952}$

& $\mathbf{0.708}$ & $\mathbf{0.707}$ & $\mathbf{0.707}$\\
\midrule

& \multicolumn{12}{c}{Runtime (seconds)}\\ 

\midrule

DCRNN & $88$ & $89$ & $89$ 

&  $89$ &  $96$ & $105$

& $45$ & $52$ & $73$

& $145$ & $146$ & $144$\\ 

TTS & $57$ & $59$ & $55$ &

$62$ & $63$ & $69$

& $28$ & $32$ & $47$

& $104$ & $121$ & $136$ \\ 

GCLSTM & $109$ & $109$ & $108$

& $120$ & $131$ & $152$ 

& $57$ & $79$ & $79$

& $178$ & $177$ & $173$\\ 






GConvGRU & $129$ & $128$ & $128$

& $148$ & $142$ & $153$

& $64$ & $73$ & $79$ 

& $179$ & $176$ & $175$\\ 

GConvLSTM & $171$ & $172$ & $172$

& $192$ & $196$ & $210$

& $86$ & $100$ & $109$ 

& $252$ & $250$ & $249$\\ 






\method~(ours) & $\mathbf{44}$ & $\mathbf{45}$ & $\mathbf{45}$

& $\mathbf{47}$ & $\mathbf{51}$ & $\mathbf{51}$

& $\mathbf{22}$ & $\mathbf{25}$ & $\mathbf{31}$

& $\mathbf{52}$ & $\mathbf{84}$ & $\mathbf{84}$

\\ 

\bottomrule
\end{tabular}
}
\end{center}
\end{table*}

\begin{table*}[t]
\caption{Forecasting comparison between \method~and previous methods in four real-world datasets in terms of the MSE metric, number of learnable parameters (\(|\Theta|\)), runtime ($t$) in minutes on $1000$ epochs (except LargeST which has $100$ epochs), and memory consumption in GB (Mem.). The best and second-best performing methods are shown in \textbf{bold} and \underline{underlined}, respectively.}
\label{Table2}
\begin{center}
\resizebox{0.96\textwidth}{!}{
\begin{tabular}{rcccccccccccccccr}
\toprule
\multirow{2}{*}{\textbf{Method}} & \multicolumn{4}{c}{\textsl{AirQuality}} & \multicolumn{4}{c}{\textsl{LargeST}} \\  \cmidrule(l){2-5} \cmidrule(l){6-9}

& MSE& \(|\Theta|\) & \(t\) & Mem. (\(\beta=0.2\)) & MSE & \(|\Theta|\) & \(t\) & Mem. (\(\beta=0.01\))\\ \midrule

TTS & $\mathbf{491.958}$ & $19.1$ K & $6.97$ & $3.03$&

$\mathbf{598.240}$& $35.1$ K  & $33.03$ & $11.90$\\

DCRNN & $530.825$ & $6.8$ K & $2.49$ & $0.86$

&  $660.086$ & $6.9$ K & $26.54$& $3.35$\\

A3TGCN & $517.683$ & $6.4$ K & $7.25$ & $3.53$

&  $732.923$ & $6.4$ K & $64.08$ & $13.86$\\

GCLSTM & $532.774$ & $4.7$ K & $4.10$& $0.75$

& $\underline{657.092}$ & $5.1$ K & $34.21$ & $2.92$\\

GConvGRU & $531.911$ & $3.5$ K & $2.62$& $0.66$

& $660.231$ & $3.8$ K & $22.93$ & $2.56$\\

GConvLSTM &  $530.077$ & $4.9$ K & $4.55$& $0.84$

& $658.892$ & $5.0$ K & $37.62$ & $3.28$\\

TGCN & $530.473$ & $6.6$ K & $7.43$ &  $1.16$

&  $679.385$ & $6.7$ K & $25.49$ & $4.53$\\ 

GGNM & $524.523$ & $4.1$ K & $60.20$ & $12.82$ 

& - & $69.4$ K & - & {OOM} \\ 

STCN & $512.471$ & $6.4$ K & $12.90$ & $0.68$ 

& $704.856$ & $6.4$ K & $53.47$ & $1.46$ \\ 

\method~(ours) & $\underline{511.092}$ & $\mathbf{12}$ & $\mathbf{0.92}$ & $\mathbf{0.05}$

& $657.387$ & $\mathbf{12}$ & $\mathbf{9.02}$ & $\mathbf{0.43}$ \\

\midrule
\midrule
\multirow{2}{*}{\textbf{Method}} & \multicolumn{4}{c}{\textsl{PeMS08}} & \multicolumn{4}{c}{\textsl{PemsBay}} \\  \cmidrule(l){2-5} \cmidrule(l){6-9}

& MSE& \(|\Theta|\) & \(t\) & Mem. (\(\beta=0.2\)) & MSE & \(|\Theta|\) & \(t\) & Mem. (\(\beta=0.1\))\\ \midrule

TTS & $\underline{468.200}$ & $14.8$ K & $4.52$ & $2.40$&

$\underline{2.905}$& $17.3$ K  & $36.04$ & $6.70$\\

DCRNN & $493.877$ & $6.8$ K & $3.08$ & $0.68$

&  $3.084$ & $6.8$ K & $18.40$& $1.89$\\

A3TGCN & $499.706$ & $6.4$ K & $7.95$ & $2.80$

&  $3.709$ & $6.4$ K & $47.79$ & $7.79$\\

GCLSTM & ${492.543}$ & $4.7$ K & $4.85$& $0.60$

& ${3.066}$ & $4.7$ K & $27.48$ & $1.65$\\

GConvGRU & $492.674$ & $3.5$ K & $2.94$& $0.51$

& $3.709$ & $3.5$ K & $19.51$ & $1.45$\\

GConvLSTM &  $493.266$ & $5.3$ K & $4.17$& $0.67$

& $3.068$ & $4.9$ K & $31.45$ & $1.85$\\

TGCN & $495.304$ & $6.6$ K & $3.41$ &  $0.92$

&  $3.082$ & $6.6$ K & $20.29$ & $2.55$\\ 

GGNM & $\textbf{434.634}$ & $1.9$ K & $40.80$ & $3.35$

& $\textbf{2.825}$ & $3.2$ K & $42.97$ & $8.56$ \\ 

STCN & $592.403$ &  $6.4$ K & $12.90$ & $0.46$ 

& $3.692$ & $6.4$ K & $13.87$ & $0.65$ \\ 

\method~(ours) & $499.851$ & $\mathbf{12}$ & $\mathbf{1.43}$ & $\mathbf{0.04}$

& $3.716$ & $\mathbf{12}$ & $\mathbf{8.07}$ & $\mathbf{0.09}$ \\

\bottomrule
\end{tabular}
}
\end{center}
\end{table*}

\textbf{Real spatiotemporal time series.}
We divide the data into the $60$\%-$20$\%-$20$\% segments for training-validation-testing.
We compare \method~with previous methods in four real datasets\footnote{Appendix \ref{AddReal} presents additional results and analysis on real datasets.}:

\textsl{AirQuality} \cite{zheng2015forecasting}: Recordings of PM 2.5 pollutant measurements collected from $437$ air quality monitoring stations across $43$ cities in China between May 2014 to April 2015.

\textsl{LargeST} \cite{liu2023largest}: Large-scale traffic forecasting dataset that contains $5$-minute traffic reading metrics recorded in a $5$-year interval from 01/01/2017 to 12/31/2021 on $8600$ traffic sensors across California. 

\textsl{PeMS08} \cite{guo2021learning}: $5$-minute traffic readings metrics for a $2$-month interval from 07/01/2016 to 08/31/2016 recorded by $170$ traffic sensors in San Bernardino.

\textsl{PEMS-BAY} \cite{li2018diffusion}: $5$-minute traffic reading metrics for a $6$-month interval from 01/01/2017 to 05/31/2017 recorded on $25$ sensors across the San Francisco Bay Area.

Table \ref{Table2} shows \method's forecasting performance compared to previous methods using the MSE metric, alongside the number of learnable parameters \(|\Theta|\), running times $t$ in minutes, and the total GPU memory consumption in Gigabytes (GB) on the \(\beta K_{test}\) test samples, where \(K_{test}\) denotes the number of test samples.
Overall, \method~demonstrates competitive MSE performance compared to previous state-of-the-art methods, while dramatically reducing memory consumption and achieving faster running times.
For instance, on the large-scale dataset \textsl{LargeST}, \method~requires approximately $28$ and $7$ times less memory than TTS and GCLSTM, respectively, while exhibiting similar MSE performance to GCLSTM.
Similarly, \method~is approximately $3.6$ times faster than TTS, and $3.8$ times faster than GCLSTM.
These results highlight the importance of \method~in the current landscape of frugality in machine learning \cite{schwartz2020green}.
Similar observations hold for the \textsl{AirQuality} dataset, where \method~ranks as the second-best method in terms of MSE.
Comparable improvements in memory consumption and running times are also evident in the \textsl{PeMS08} and \textsl{PemsBay} datasets.


We conduct an additional experiment in the \textsl{AirQuality} dataset where we vary the percentage of training data used to train some models.
We report results for A3TGCN, TTS, and \method, as illustrated in Figure \ref{fig:LimitedData}.
The results show three regimes, i) low-data regime (less than $40\%$ of training data), ii) transition phase (between $40\%$ and $60\%$), and iii) abundant training data (more than $60\%$).
Methods with low or medium-sized parameter budgets like A3TGCN and \method~perform the best in the low-data regime.
Subsequently, we see a transition phase where TTS, with a high parameter budget, becomes the best method when abundant training data is available.
Figure \ref{fig:LimitedData} gives important hints as to whether to use \method~or TTS regarding the specific scenarios of some potential real-world problem in spatiotemporal forecasting.
Thus, from Table \ref{Table2} and Figure \ref{fig:LimitedData} we argue that \method~is an important alternative to other spatiotemporal methods when we have applications with limited data and resources.

        










\textbf{Experimental stability analysis.}
We conduct a comprehensive experiment to validate the theoretical stability analysis of \method~as detailed in Section \ref{sec:stability_analysis}.
We generate Erd\H{o}s-R\'enyi graphs with \(N=100\) and different values of sparsity \(p\in\{0.1,0.3,\dots,0.9\}\) for the adjacency matrix \(\mathbf{A}\) and the graph filter coefficients \(\boldsymbol{\theta}\), where higher $p$ indicates less sparsity.
We fix $M=10$ for the graph filters in all experiments.
Therefore, we perturb the input graph with a Gaussian-generated perturbation matrix \(\mathbf{E}\) in (\ref{pert_model}) with varying SNRs in \(\{-15,0,15\}\) dB, corresponding to the upper bound \(\delta_{\mathbf{A}}\) in (\ref{pert_model}).
Figure \ref{StabFig} illustrates the averaged rMSE over five independent realizations per setting, indicating that forecasting performance becomes more stable with increasing SNR values (decreasing \(\delta_{\mathbf{A}}\)).
Similarly, we observe that \method~performs better with higher sparsity values.
Consequently, even in cases of low SNR (high noise), \method~can leverage the sparsity of the underlying process to its advantage.



\begin{figure}
    \centering
    \begin{minipage}{.45\textwidth}
      \centering
      \includegraphics[width=.85\textwidth]{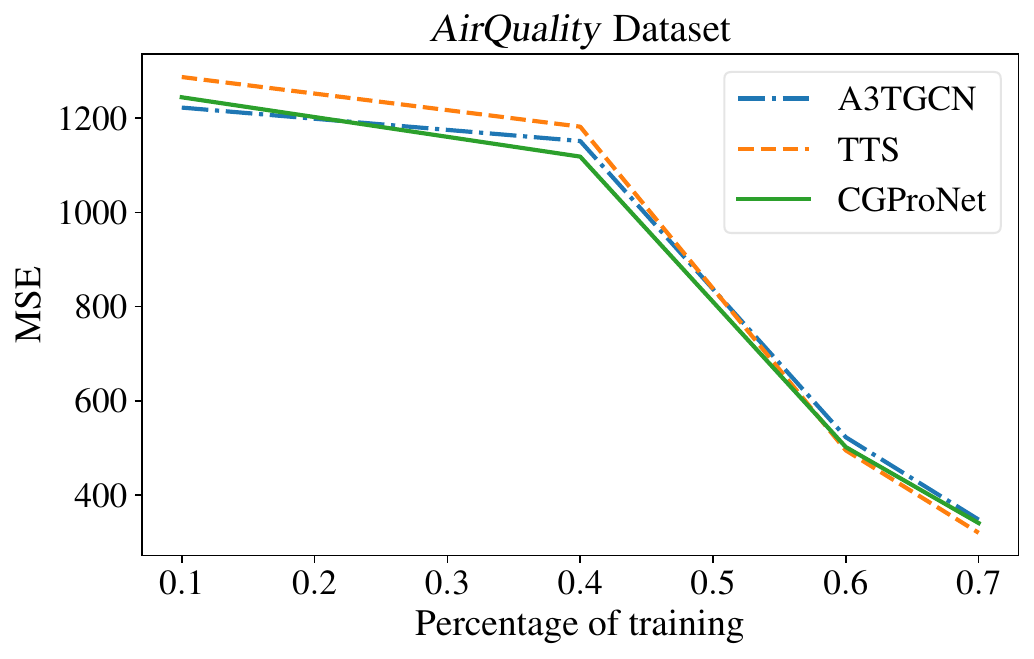}
      \caption{Forecasting performance comparison in the case of limited training data.}
      \label{fig:LimitedData}
    \end{minipage}
    \hfill
    \begin{minipage}{.45\textwidth}
      \centering
      \includegraphics[width=.85\textwidth]{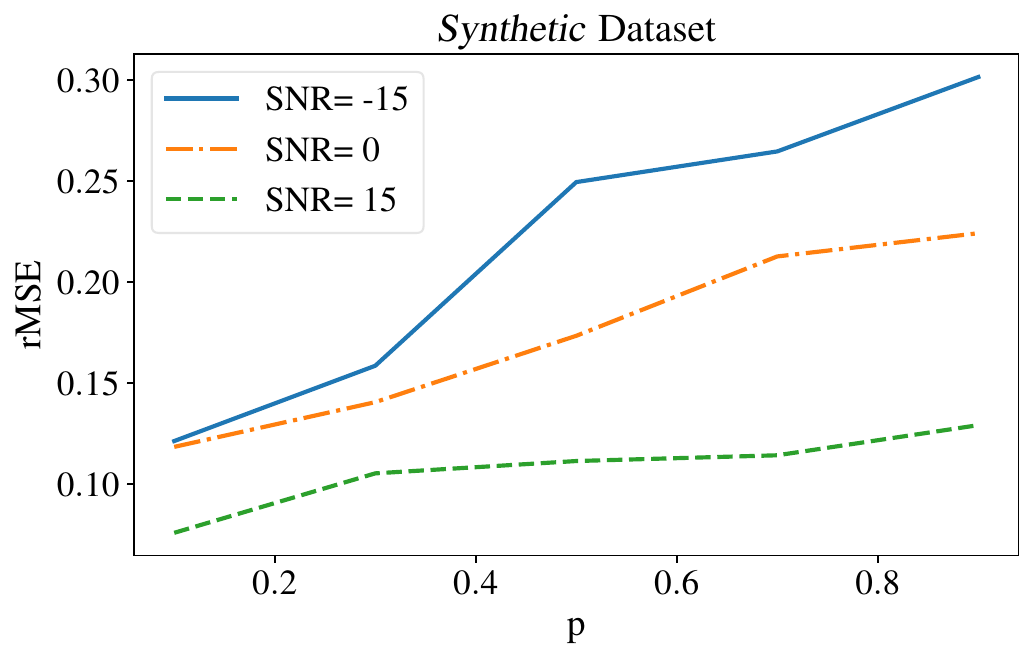}
      \caption{The averaged rMSE vs. sparsity $p$ with different SNRs for the synthetic data.}
      \label{StabFig} 
    \end{minipage}
\end{figure}

\begin{table*}
    \caption{\color{black} Multiple horizons forecasting comparison between different extensions of \method~and the TTS method on the \textsl{AirQuality} dataset.}
    \label{Table_Multi_H}
    \centering
    \resizebox{\textwidth}{!}{
    \begin{tabular}{rcccccccccccc}
        \toprule
        \multirow{2}{*}{\textbf{Method}} & \multicolumn{4}{c}{$H=3$} & \multicolumn{4}{c}{$H=6$} & \multicolumn{4}{c}{$H=9$} \\
        \cmidrule(l){2-5} \cmidrule(l){6-9} \cmidrule(l){10-13}
        & MSE & $|\Theta|$ & $t$ & Mem. & MSE & $|\Theta|$ & $t$ & Mem. & MSE & $|\Theta|$ & $t$ & Mem. \\
        
         \midrule

TTS&$\mathbf{803.452}$&$19.2$ K&$13.50$&$3.03$ &$1142.317$&$19.4$ K&$13.29$&$3.04$ 
&${1344.903}$&$19.5$ K& $13.10$&$3.05$\\

\method$_\text{MLP}$&$1090.712$&$\mathbf{15}$&$\mathbf{0.76}$ & $\mathbf{0.05}$ &$1306.661$&$\mathbf{18}$&$\mathbf{0.88}$&$\mathbf{0.06}$ &$1469.130$& $\mathbf{21}$ &$\mathbf{1.04}$&$\mathbf{0.07}$\\

\method$_\text{Ada}$&${804.710}$& $30$ & $1.49$ &$0.13$ &$\mathbf{1132.528}$&$57$&$2.70$&${0.25}$ &$\mathbf{1336.466}$&$84$&${3.82}$&${0.37}$\\
\bottomrule
    \end{tabular}
    }
\end{table*}

{
\color{black}

\textbf{Multiple horizons forecasting.}
Table \ref{Table_Multi_H} shows the results of multiple horizons forecasting for $H=\{3,6,9\}$ of different extensions of \method~and the TTS method on the \textsl{AirQuality} dataset.
We observe that \method$_\text{MLP}$ shows high efficiency in terms of memory and computational complexity for all horizons, although with reduced forecasting performance.
On the other hand, \method$_\text{Ada}$ shows high forecasting accuracy, outperforming TTS for $H=\{6,9\}$.
The high precision of \method$_\text{Ada}$ comes at a slightly higher computational cost due to the recursion in \eqref{eqn:CGProNet_Adaptive}.

}





\begin{wraptable}{r}{0.43\textwidth}
    \vspace{-10pt}
    \caption{Ablation study of the choice of regularization and non-linearity in \method.}
    \label{tab:ablation_study}
    \centering
    \resizebox{0.43\textwidth}{!}{
    \begin{tabular}{ccccc}
        \toprule
        $P(\mathbf{A}, \boldsymbol{\theta}_i)$ & $\Vert \boldsymbol{\theta} \Vert_1$ & $\Vert \boldsymbol{\theta} \Vert_2$ & $\sigma(\cdot)$ & MSE \\
         \midrule
         \ding{55} & \ding{55} & \ding{55} & \ding{55} &  $1.067$\\
         \ding{51} & \ding{55} & \ding{55} & \ding{55} &  $0.967$\\
         \ding{51} & \ding{51} & \ding{55} & \ding{55} & $0.851$\\
         \ding{51} & \ding{55} & \ding{51} & \ding{55} & $0.849$\\
         \ding{51} & \ding{55} & \ding{55} & \ding{51} & $0.726$\\
         \ding{51} & \ding{55} & \ding{51} & \ding{51} & $0.720$\\
         \ding{51} & \ding{51} & \ding{55} & \ding{51} & $\mathbf{0.716}$\\
         \bottomrule
    \end{tabular}
    }
\end{wraptable} 

\textbf{Ablation study.}
We conduct an ablation study to analyze several aspects of the loss function and the architecture of \method.
We generate ten realizations of the synthetic data with \(N=30\), $\text{SNR}=0$, and \(p=0.03\).
We analyse \method~with $\ell_1$ $(\Vert \boldsymbol{\theta} \Vert_1)$ or $\ell_2$ $(\Vert \boldsymbol{\theta} \Vert_2)$ regularization into our loss function $\mathcal{L}$.
Similarly, we conduct experiments i) using $\tanh(\cdot)$ as non-linearity or a linear function in our aggregation scheme, and ii) using graph filters $P(\mathbf{A}, \boldsymbol{\theta}_i)$ or unconstrained matrix coefficients.
Table \ref{tab:ablation_study} shows the average MSE over seven different settings.
The VAR model (first row), which uses unconstrained matrix coefficients, performs poorly in forecasting due to its lack of non-linearities, higher parameter count, and absence of relational inductive bias.
We observe that adding either $\ell_1$ or $\ell_2$ regularization improves the results of \method~independently of the activation function $\sigma(\cdot)$ used.
We also observe that $\ell_1$ is superior to $\ell_2$ regularization in any case, highlighting the results from the stability analysis for high sparsity of the underlying process.
Finally, combining non-linearity with $\ell_1$ regularization leads to the best results, highlighting its importance in \method.



\section{Conclusion {\color{black} and Limitations}}
\label{sec:conclusions}

In this paper, we introduced a non-linear framework for spatiotemporal forecasting using concepts of CGPs, named Causal Graph Process Neural Network (\method) model.
Our model dramatically reduces the computational and memory complexity concerning the state-of-the-art GNN-based models.
To that end, \method~relies on higher-order graph filters, optimizing the model with fewer parameters, reducing memory usage, and enhancing runtime.
We also comprehensively studied the theoretical stability properties of \method.
Extensive experimental results on synthetic and real-world spatiotemporal graph signals validated the potential impact of \method~in the current landscape of spatiotemporal forecasting. 
In future work, we will explore adapting a broader class of graph filters to make the proposed framework more general.


\section{Acknowledgement}
This work was supported by the center Hi! PARIS and ANR (French National Research Agency) under the JCJC project GraphIA (ANR-20-CE23-0009-01).

\bibliographystyle{ieeetr}
\bibliography{main}


\newpage

\appendix



\newpage
\section*{Appendix}

The appendix is organized as follows: Section \ref{app:VAR_CGP} discusses the cases in the proposed framework that can coincide with the classic VAR and CGP models.
The proofs of the stated propositions and theorems in the main body of the paper are provided by Section \ref{app:proofs}. Further theoretical and experimental discussions about the possible extension of the proposed framework to the case of using continuous (heat kernel) polynomial graph filters are outlined in Section \ref{C2GProNet_det}.
\textcolor{black}{Section \ref{app:additional_multiH_forecasting} presents one additional extension of \method~to forecast multiple horizons with some results.}
The details about the generation of synthetic spatiotemporal time series, and additional results on the other kinds of graph structures (\ie Directed Stochastic Block Models (DSBMs) can be found in Section \ref{app:synthetic_dataset}. 
Finally, Section \ref{AddReal} describes the details about the hyperparameters of the real-world experiments and additional results on real-world datasets.

\section{Recovering VAR and CGP}
\label{app:VAR_CGP}

\subsection{Vector Auto-Regressive (VAR) process}

By considering $\{\alpha_i=1\}_{i=1}^M$, $\sigma(\cdot)$ as a linear function, and also $\texttt{GF}_{\mathcal{G},\boldsymbol{\phi}_i}=\mathbf{R}_i\in\mathbb{R}^{N\times N}$, where \(\{\mathbf{R}_i\}_{i=1}^M\) are unconstrained coefficient matrices, the proposed framework coincides with the well-known VAR model. In the VAR process, the causal dependencies between different time samples with an \(M\)-sample window can be defined using coefficient matrices \(\{\mathbf{R}_i\}_{i=1}^M\) as follows:
\begin{equation}
\label{VAR}
\mathbf{x}_{k} = \sum_{i=1}^{M}{\mathbf{R}_{i}\mathbf{x}_{k-i}}+\mathbf{w}_{k}\text{,}
\end{equation}
with \(\mathbf{w}_{k}\) being the instantaneous exogenous noise.
In addition to the numerous advantages of this model, there are some serious limitations to using this type of modeling in real-world problems. For example, the number of learnable parameters is \(MN^2\) which makes it prohibitive to be used in the case of large network size involved or if the AR order \(M\) is considerably large. 

\subsection{Causal Graph Process (CGP)}

The choices of $\{\alpha_i=1\}_{i=1}^M$, $\sigma(\cdot)$ as a linear function, and 
\begin{equation}
\label{Def_GPi2}
\texttt{GF}_{\mathcal{G},\boldsymbol{\phi}_i}=\sum_{j=0}^{i}{\phi_{ij}\mathbf{A}^j},
\end{equation}
leads to the definition of the CGP model by \cite{mei2016signal} which was proposed to address the issues with the VAR model (\ref{VAR}), to benefit from the GSP interpretations and also significantly reduce the number of learnable parameters. The model is developed by replacing the unconstrained VAR coefficient matrices \(\{\mathbf{R}_i\}_{i=1}^M\) with the polynomial graph filters.

\section{Proofs}
\label{app:proofs}
\subsection{Proof of Proposition \ref{prop1}}
\label{Proof_prop1}

\begin{proof}
By inserting the perturbation models into the formation of \(P(\hat{\mathbf{A}},\hat{\boldsymbol{\theta}}_i)\), we obtain:
\begin{equation}
P(\hat{\mathbf{A}},\hat{\boldsymbol{\theta}}_i)=\sum_{j=0}^{i}{\hat{\theta}_{ij}\hat{\mathbf{A}}^j}=\sum_{j=0}^{i}{\theta_{ij}\hat{\mathbf{A}}^j} + \sum_{j=0}^{i}{z_{ij}\hat{\mathbf{A}}^j}.
\end{equation}
Next, one can write:
\begin{equation}
\begin{split}
&P(\hat{\mathbf{A}},\hat{\boldsymbol{\theta}}_i)\le\sum_{j=0}^{i}{|\theta_{ij}|\|\hat{\mathbf{A}}^j\|_2} + \sum_{j=0}^{i}{|z_{ij}|\|\hat{\mathbf{A}}^j\|_2}.
\end{split}
\end{equation}

The next lemma helps to find an upper bound on \(\|\hat{\mathbf{A}}^j\|_2\).
\begin{lemma}
\label{lemma2} 
With the assumptions and definitions of Proposition \ref{prop1}, it holds that:
\begin{equation}
\label{Aj2}
\begin{split}
\|\hat{\mathbf{A}}^j\|_2\le 
L +\delta_\mathbf{A}\hat{L}_M(\delta_\mathbf{A}).
\end{split}
\end{equation}
\end{lemma}

\begin{proof}
Using the triangularity principle for inequalities, one can write:
\begin{equation}
\label{Aj}
\begin{split}
\|\hat{\mathbf{A}}^j\|_2\le \left\|\mathbf{A}^j + \left(\hat{\mathbf{A}}^j-\mathbf{A}^j\right)\right\|_2&\le\|\mathbf{A}^j\|_2+\|\hat{\mathbf{A}}^j-\mathbf{A}^j\|_2.
\end{split}
\end{equation}
Then, the next lemma helps to write the upper bound of the equality (\ref{Aj}) in terms of the stated definitions.

\begin{lemma}{Eq. (35) in \cite{mei2016signal}.}
\label{lemma1}
With the assumptions and definitions of Proposition \ref{prop1}, it holds that:
\begin{equation}
\max_{1\le i \le M}{\|\hat{\mathbf{A}}^i-\mathbf{A}^i\|_2}\le \delta_\mathbf{A}\hat{L}_M(\delta_\mathbf{A}).
\end{equation}
\end{lemma}

Therefore, using Lemma \ref{lemma1}, the \(\|\hat{\mathbf{A}}^j\|_2\) in (\ref{Aj}) is upper-bounded as:
\begin{equation}
\label{Aj3}
\|\hat{\mathbf{A}}^j\|_2\le 
L +\delta_\mathbf{A}\hat{L}_M(\delta_\mathbf{A}).
\end{equation}
\end{proof}
Finally, using Lemma \ref{lemma2}, the desired bound stated in Proposition \ref{prop1} is obtained as:
\begin{equation}
\begin{split}
&\left\|P(\hat{\mathbf{A}},\hat{\boldsymbol{\theta}}_i)\right\|_2\le (\rho_{\boldsymbol{\theta}_i}+\delta_{\boldsymbol{\theta}_i})(L+\delta_{\mathbf{A}}\hat{L}_M(\delta_\mathbf{A})).
\end{split}
\end{equation}
\end{proof}

\subsection{Proof of Proposition \ref{prop2}}
\label{Proof_prop2}
\begin{proof}
By plugging the definition (\ref{Def_GPi}) into the left part of (\ref{GF_stab}), one can write:
\begin{equation}
\resizebox{\hsize}{!}{$\left\|P(\hat{\mathbf{A}},\hat{\boldsymbol{\theta}}_i)-P(\mathbf{A},\boldsymbol{\theta}_i)\right\|_2=\left\|\sum_{j=0}^{i}{\hat{\theta}_{ij}\hat{\mathbf{A}}^j}-\sum_{j=0}^{i}{\theta_{ij}\mathbf{A}^j}\right\|_2\le\sum_{j=0}^{i}{|\theta_{ij}|\|\hat{\mathbf{A}}^j-\mathbf{A}^j\|_2}+\sum_{j=0}^{i}{|z_{ij}|\|\hat{\mathbf{A}}^j\|_2},$}
\end{equation}
where the triangularity principle in inequalities was used. Then, using Lemmas (\ref{lemma2}) and (\ref{lemma1}):
\begin{equation}
\begin{split}
\left\|P(\hat{\mathbf{A}},\hat{\boldsymbol{\theta}}_i)-P(\mathbf{A},\boldsymbol{\theta}_i)\right\|_2&\le \|\boldsymbol{\theta}_i\|_1\delta_{\mathbf{A}}\hat{L}_M(\delta_\mathbf{A})+\|\boldsymbol{\theta}_i-\hat{\boldsymbol{\theta}}_i\|_1(L+\delta_{\mathbf{A}}\hat{L}_M(\delta_\mathbf{A}))\\
&\le \rho_{\boldsymbol{\theta}_i}\delta_{\mathbf{A}}\hat{L}_M(\delta_\mathbf{A}) + \delta_{\boldsymbol{\theta}_i}(L+\delta_{\mathbf{A}}\hat{L}_M(\delta_\mathbf{A})),
\end{split}
\end{equation}
which gets the desired results stated in Proposition \ref{prop2}.
\end{proof}

\subsection{Proof of Theorem \ref{thm1}}
\label{proof_thm1}
\begin{proof}
By plugging the prediction model \eqref{eqn:CGProNet_GNN} into the prediction difference, one can write the prediction difference as:
\begin{equation}
\begin{split}
&\|\tilde{\mathbf{x}}_{k}-\mathbf{x}_{k}\|_2=\|\sum_{i=1}^M{\hat{\alpha}_i \sigma \left(P(\hat{\mathbf{A}},\hat{\boldsymbol{\theta}}_i)\mathbf{x}(k-i)\right)-\alpha_i \sigma \left(P(\mathbf{A},\boldsymbol{\theta}_i)\mathbf{x}(k-i)\right)}\|_2\\
&\resizebox{\hsize}{!}{$ \le\|\sum_{i=1}^M{\alpha_i \left(\sigma (P(\hat{\mathbf{A}},\hat{\boldsymbol{\theta}}_i)\mathbf{x}(k-i))-\sigma \left(P(\mathbf{A},\boldsymbol{\theta}_i)\mathbf{x}(k-i)\right)\right)}\|_2+\|\sum_{i=1}^M{e_i \left(\sigma(P(\hat{\mathbf{A}},\hat{\boldsymbol{\theta}}_i)\mathbf{x}(k-i))\right)}\|_2.$}
\end{split}
\end{equation}
Next, due to \(\sigma(0)=0\) and \(\sigma(\cdot)\) being Lipschitz for \(\forall i\), we have:
\begin{equation}
\|\sigma(\mathbf{x})\|_2\le\|\mathbf{x}\|_2;\: \forall~i,\: \text{and for any}\: \mathbf{x}\text{.}
\end{equation}
Therefore,
\begin{equation}
\begin{split}
&\|\sum_{i=1}^M{\alpha_i \left(\sigma(P(\hat{\mathbf{A}},\hat{\boldsymbol{\theta}}_i)\mathbf{x}(k-i))-\sigma\left(P(\mathbf{A},\boldsymbol{\theta}_i)\mathbf{x}(k-i)\right)\right)}\|_2+\|\sum_{i=1}^M{e_i \left(\sigma(P(\hat{\mathbf{A}},\hat{\boldsymbol{\theta}}_i)\mathbf{x}(k-i))\right)}\|_2\\
&\le\sum_{i=1}^{M}{|\alpha_i|\|P(\hat{\mathbf{A}},\hat{\boldsymbol{\theta}}_i)-P(\mathbf{A},\boldsymbol{\theta}_i)\|_2 \|\mathbf{x}(k-i)\|_2}+\sum_{i=1}^{M}{|e_i|\|P(\hat{\mathbf{A}},\hat{\boldsymbol{\theta}}_i)\|_2 \|\mathbf{x}(k-i)\|_2}\\
&\resizebox{\hsize}{!}{$\le\sum_{i=1}^{M}{|\alpha_i|\left(\rho_{\boldsymbol{\theta}_i}\delta_{\mathbf{A}}\hat{L}_M(\delta_\mathbf{A}) + \delta_{\boldsymbol{\theta}_i}(L+\delta_{\mathbf{A}}\hat{L}_M(\delta_\mathbf{A}))\right) \|\mathbf{X}\|_{1,2}}+\sum_{i=1}^{M}{|e_i|(\rho_{\boldsymbol{\theta}_i}+\delta_{\boldsymbol{\theta}_i})(L+\delta_{\mathbf{A}}\hat{L}_M(\delta_\mathbf{A})) \|\mathbf{X}\|_{1,2}}$}\\
&\le\rho_{\boldsymbol{\alpha}}\left(\rho_{\boldsymbol{\theta}}\delta_{\mathbf{A}}\hat{L}_M(\delta_\mathbf{A}) + \delta_{\boldsymbol{\theta}}(L+\delta_{\mathbf{A}}\hat{L}_M(\delta_\mathbf{A}))\right) \|\mathbf{X}\|_{1,2}+\delta_{\boldsymbol{\alpha}}(\rho_{\boldsymbol{\theta}}+\delta_{\boldsymbol{\theta}})(L+\delta_{\mathbf{A}}\hat{L}_M(\delta_\mathbf{A})) \|\mathbf{X}\|_{1,2},
\end{split}
\end{equation}
where the second inequality relies on the fact that \(\forall~i:\|\mathbf{x}(i)\|_2\le\|\mathbf{X}\|_{1,2}\), and this concludes the proof.
\end{proof}

\section{Continuous \method~(\cmethod)}
\label{C2GProNet_det}


The number of parameters in \method~could be reduced by adopting \textit{continuous} polynomial graph filters.
More precisely, we can use the heat kernels as follows:
\begin{equation}
\begin{split}
\texttt{GF}_{\mathcal{G},\boldsymbol{\theta}_i}=\theta_{i1}\exp\left( \theta_{i2}\mathbf{A} \right)=\sum_{j=0}^{\infty}{\left(\frac{\theta_{i1}(\theta_{i2})^j}{j!}\right)\mathbf{A}^j},
\end{split}
\end{equation}
where \( \boldsymbol{\theta}_i=[\theta_{i1}, \theta_{i2}]^\top\in\mathbb{R}^{2} \) are learnable parameters.
Therefore, the GNN model is given by:
\begin{equation}
    \label{eqn:C2GProNet}
    \tilde{\mathbf{x}}_k = \sum_{i=1}^M{\alpha_i  \tanh\left( \theta_{i1}\exp\left( \theta_{i2}\mathbf{A} \right)  \mathbf{x}_{k-i}  \right)}.
\end{equation}
Notice that the number of learnable parameters in \eqref{eqn:C2GProNet} is given by \(3M\), which is a significant reduction concerning the non-continuous alternative.
However, we need the eigenvalue decomposition of \(\mathbf{A}\) to compute $\exp\left( \theta_{i2}\mathbf{A} \right)$ in \eqref{eqn:C2GProNet}, which has a computational complexity of \(\mathcal{O}(N^3)\).
Thus, the continuous \method~decreases memory footprint at the expense of increased computational complexity.

It is worth noting that the continuous \method~can propagate global and local information within each time step with flexible and learnable parameters.
Therefore, the efficient receptive field on the neighbor nodes is automatically optimized through the training process, alleviating the over-smoothing and over-squashing issues \cite{sharp2022diffusionnet,behmanesh2023tide,giraldo2023trade}.
Further theoretical and experimental analysis of the continuous \method~model can be found in Appendix \ref{C2GProNet_det}.

From another point of view, it seems that we can extend the current approach to the continuous space by replacing the discrete graph filters with heat kernels as follows:
\begin{equation}
\label{CGP_GNN2}
\begin{split}
\mathbf{x}_{k}&=\sum_{i=1}^M{\alpha_i \sigma\left(c_i \Psi(\mathbf{A},t_i)\mathbf{x}(k-i)\right)} + \mathbf{w}(k),\\
\end{split}
\end{equation}
where
\begin{equation}
\label{heat_ker}
\Psi(\mathbf{A},t_i)=e^{\mathbf{A}t_i}=\sum_{j=0}^{\infty}{\frac{(t_i\mathbf{A})^j}{j!}}\text{,}
\end{equation}
which \(\{t_i\}_{i=1}^M\) are learnable and \(\{t_i\}_{i=1}^M\) are the learnable graph filter coefficients. Using this approach, both the global and local information can be propagated within each time step with flexible and learnable parameters \(\{t_i\}_{i=1}^M\) \cite{sharp2022diffusionnet,behmanesh2023tide}.

The main challenge here is the need for performing an EVD on \(\mathbf{A}\) (\(\mathcal{O}(N^3)\)) that can be precomputed only once as a preprocessing step. Therefore, exploiting the proposed framework (\ref{CGP_GNN2}) reduces the number of learnable parameters to \(3M\) with the expense of increasing the computational complexity, especially for large graphs.

The rigorous investigation of such extensions is targeted in our future work.

\subsection{Stability Analysis of \cmethod}
\label{app:stability_C2GProNet}

The next lemma provides the upper bound for the difference between true and perturbed heat graph filters:

\begin{lemma}
\label{Lemma_app}
With the perturbation model (\ref{pert_model}), for any \(i=1,\hdots,M\), the stability of the heat kernel (\ref{heat_ker}) can be stated as:
\begin{equation}
\|\Psi(\hat{\mathbf{A}},t_i)-\Psi(\mathbf{A},t_i)\|\le\left(t_i \|\mathbf{E}\| e^{\left(\mu(\mathbf{A})-\alpha(\mathbf{A})+\|\mathbf{A}\|+\|\mathbf{E}\|\right)t_i}\right),
\end{equation}
where \(\|\cdot \|\) states the spectral norm, \(\lambda(\mathbf{A})\) is the set of eigenvalues of \(\mathbf{A}\), \(Re(\cdot)\) obtained the real part of a complex argument, and
\begin{equation}
\label{heat_stab}
\begin{split}
&\Psi(\mathbf{A},t)=e^{\mathbf{A}t}\\
&\alpha(\mathbf{A})=\max\left\{\text{Re}(\lambda)|\lambda\in\lambda(\mathbf{A})\right\}\\
&\mu(\mathbf{A})=\left\{\mu|\mu\in\lambda\left(\frac{\mathbf{A}+\mathbf{A}^\top}{2}\right)\right\}.\\
\end{split} 
\end{equation}
\end{lemma}

\begin{proof}
Firstly, we note that by considering the expansion
\begin{equation}
\label{heat_ker_sum}
e^{\mathbf{A}t}=\sum_{j=0}^{\infty}{\frac{(t\mathbf{A})^j}{j!}}\text{,}
\end{equation}
one can trivially obtain \cite{van1977sensitivity}:
\begin{equation}
\label{eq1}
\|e^{\mathbf{A}t_i}\|\le e^{\|\mathbf{A}\|t_i}.
\end{equation}

Then, following the Theorem 2 in \cite{van1977sensitivity}, we can write:
\begin{equation}
\label{eq2}
\frac{\|e^{(\mathbf{A}+\mathbf{E})t_i}-e^{\mathbf{A}t_i}\|}{\|e^{\mathbf{A}t}\|}\le t_i \|\mathbf{E}\| e^{(\mu(\mathbf{A})-\alpha(\mathbf{A})+\|\mathbf{E}\|)t_i}.
\end{equation}
Using (\ref{eq1}) and (\ref{eq2}), the stated results in Lemma (\ref{Lemma_app}) is obtained.
\end{proof}
\begin{theorem}
With the assumptions of the non-linearity function $\sigma(\cdot)$ being Lipschitz for \(\forall i\), \(\sigma(0)=0\) for \(\forall i\), error matrix \(\mathbf{E}=\hat{\mathbf{A}}-\mathbf{A}\), graph filter coefficients \(\boldsymbol{\theta}\), heat graph filter coefficients \(\mathbf{t}\), and also mixing coefficients \(\boldsymbol{\alpha}\) are respectively upper-bounded as \(\|\boldsymbol{\theta}\|_1\le\rho_{\boldsymbol{\theta}}\), \(\|\mathbf{t}\|_1\le\rho_\mathbf{t}\), and \(\|\boldsymbol{\alpha}\|_1\le\rho_{\boldsymbol{\alpha}}\), the difference between the true and predicted outputs of the proposed \method~ for the current time step \(k\), \ie \(\mathbf{x}_{k}\) and \(\tilde{\mathbf{x}}_{k}\), respectively, is upper-bounded as:
\begin{equation}
\label{stab_bound_heat}
\begin{split}
&\|\tilde{\mathbf{x}}_{k}-\mathbf{x}_{k}\|_2\le M\rho_{\boldsymbol{\alpha}}\rho_{\boldsymbol{\theta}}\rho_{\mathbf{t}} \left(e^{\left(\mu(\mathbf{A})-\alpha(\mathbf{A})+\|\mathbf{A}\|+\delta_\mathbf{A}\|\right)\rho_\mathbf{t}}\right) \|\mathbf{X}\|_{1,2}.
\end{split}
\end{equation}
\end{theorem}

\begin{proof}
By plugging the model (\ref{CGP_GNN2}) in the difference expression, we have:
\begin{equation}
\begin{split}
&\|\tilde{\mathbf{x}}_{k}-\mathbf{x}_{k}\|_2\\
&\resizebox{\hsize}{!}{$=\left\|\sum_{i=1}^M{\alpha_i\left(\sigma(c_i\Psi(\hat{\mathbf{A}},t_i)\mathbf{x}(k-i))-\sigma\left(c_i\Psi(\mathbf{A},t_i)\mathbf{x}(k-i)\right)\right)}\right\|_2\le\sum_{i=1}^{M}{|\alpha_i c_i| \|\Psi(\hat{\mathbf{A}},t_i)-\Psi(\mathbf{A},t_i)\|_2\|\mathbf{x}(k-i)\|_2}$}.\\
\end{split}
\end{equation}
Next, Lemma \ref{Lemma_app} provides the upper bound for the stability of heat graph filters, and the previous inequality takes the form of:
\begin{equation}
\begin{split}
&\|\tilde{\mathbf{x}}_{k}-\mathbf{x}_{k}\|_2\\
&\resizebox{\hsize}{!}{$\le\sum_{i=1}^{M}{|\alpha_i c_i| \left(t_i \|\mathbf{E}\| e^{\left(\mu(\mathbf{A})-\alpha(\mathbf{A})+\|\mathbf{A}\|+\|\mathbf{E}\|\right)t_i}\right) \|\mathbf{X}\|_{1,2}}\le M \rho_{\boldsymbol{\alpha}}\rho_{\boldsymbol{\theta}}\rho_{\mathbf{t}}\delta_\mathbf{A} \left(e^{\left(\mu(\mathbf{A})-\alpha(\mathbf{A})+\|\mathbf{A}\|+\delta_\mathbf{A}\|\right)\rho_\mathbf{t}}\right) \|\mathbf{X}\|_{1,2}.$}
\end{split}
\end{equation}   
\end{proof}
As can be seen in stability bound (\ref{stab_bound_heat}), similar to the stated deductions from (\ref{stab_bound}), the proposed heat model (\ref{CGP_GNN2}) can effectively benefit from the sparsity of the underlying process (due to the presence of \(\rho_{\boldsymbol{\alpha}}\), \(\rho_{\boldsymbol{\theta}}\), and \(\rho_{\boldsymbol{\alpha}}\) in the upper bound (\ref{stab_bound_heat})). Besides, the smaller the AR order \(M\), the more stable the network. Note that in order to the heat graph filters to be stationary they must hold \cite{van1977sensitivity}:
\begin{equation}
\lim_{t\rightarrow\infty}{e^{\mathbf{A}t}}=\mathbf{0}\Leftrightarrow \alpha(\mathbf{A})<0.
\end{equation}
\subsection{Comparison of \method~ and \cmethod~}
\label{app:discrete_vs_continous_results}

Table \ref{Table_Heat} provides the memory consumption comparison between \method~and \cmethod~over the real-world spatiotemporal datasets across increasing the AR order \(M\). Generally, from these results, it can be seen that \cmethod~has a considerably better memory footprint on small to medium-sized network data with more robustness against increasing the AR order \(M\). On the other hand, due to the usage of eigenvalue decomposition for evaluating the heat kernels, \cmethod~is not a very good choice for pressing spatiotemporal time-series on large graphs, \ie LargeST and PvUS datasets.

\begin{table}[t]
\caption{The memory consumption comparison between \method~and \cmethod~over the real-world spatiotemporal datasets across increasing the AR order \(M\in\{3,6,12\}\).}
\label{Table_Heat}
\begin{center}
\begin{small}
\begin{tabular}{rcccccccccccccccr}
\toprule

$M$ & $3$ & $6$ & $12$ & $3$ & $6$ & $12$ & $3$ & $6$ & $12$ & $3$ & $6$ & $12$  \\ \midrule

& \multicolumn{3}{c}{\textsl{AirQuality}} & \multicolumn{3}{c}{\textsl{MetrLA}} & \multicolumn{3}{c}{\textsl{PeMS04}} & \multicolumn{3}{c}{\textsl{PeMS07}} \\ [0.5ex] 
\cmidrule(l){2-4} \cmidrule(l){5-7} \cmidrule(l){8-10} \cmidrule(l){11-13}

\method~ & $0.04$ & $0.10$ & $0.29$ &

$0.07$ & $0.18$ & $0.54$

& $0.06$ & $0.14$ & $0.41$ 

& $0.03$ & $0.08$ & $0.48$\\ 

\cmethod~ & $0.04$ & $0.07$ & $0.12$ &

$0.06$ & $0.11$ & $0.21$

& $0.05$ & $0.08$ & $0.16$ 

& $0.03$ & $0.05$ & $0.22$
\\ 

\midrule

& \multicolumn{3}{c}{\textsl{PeMS08}} & \multicolumn{3}{c}{\textsl{PemsBay}} & \multicolumn{3}{c}{\textsl{LargeST}} & \multicolumn{3}{c}{\textsl{PvUS}} \\ [0.5ex] 
\cmidrule(l){2-4} \cmidrule(l){5-7} \cmidrule(l){8-10} \cmidrule(l){11-13}

\method~ & $0.08$ & $0.22$ & $0.23$ &

$0.42$ & $0.66$ & $0.67$

& $0.06$ & $0.14$ & $1.47$ 

& $0.13$ & $0.15$ & $0.21$\\ 

\cmethod~ & $0.07$ & $0.13$ & $0.09$ &

$1.22$ & $2.15$ & $0.25$

& $0.05$ & $0.08$ & $4.01$ 

& $0.39$ & $0.68$ & $1.26$
\\ 

\bottomrule
\end{tabular}
\end{small}
\end{center}
\vskip -0.1in
\end{table}

\section{Addtional Extension for Multiple Horizon Forecasting}
\label{app:additional_multiH_forecasting}

\textbf{Adaptive prediction with shared weights.}
We use \eqref{eqn:recursion} for the extension of the proposed method to forecast multiple horizons with shared weights, denoted as \method$_{\text{Sha}}$, where the \(h\)-th predicted time sample (for $h=1,\hdots,H$) can be expressed as:
\begin{equation}
\tilde{\mathbf{x}}_{k+h-1}=f_{M}\left(\mathbf{X}_k^{(h)};\mathbf{A},\{\boldsymbol{\theta}_i\}_{i=1}^{M},\{\alpha_i\}_{i=1}^{M}\right).
\end{equation}
where $\mathbf{X}_k^{(1)}=[\mathbf{x}_{k-M},\hdots,\mathbf{x}_{k-1}]\in\mathbb{R}^{N\times M}$, and $\mathbf{X}_k^{(h>1)}=[\mathbf{x}_{k-M+h-1},\hdots,\mathbf{x}_{k-1},\tilde{\mathbf{x}}_{k},\hdots,\tilde{\mathbf{x}}_{k+h-2}]\in\mathbb{R}^{N\times M}$.
Finally, the loss function is considered between the true and predicted $H$ consequent time samples as $\mathcal{L}(\{\mathbf{x}_{k+h-1}\}_{h=1}^H,\{\tilde{\mathbf{x}}_{k+h-1}\}_{h=1}^H)$.
Precisely, in this approach and for $h>1$, the outputs of the network for $h-1$ previous time samples are included in the prediction process of the $h$-th time sample. 
Table \ref{Table_Multi_H_2} show the results for TTS, \method$_{\text{MLP}}$, \method$_{\text{Sha}}$, and \method$_{\text{Ada}}$.
We observe that \method$_{\text{Sha}}$ presents very competitive results regarding TTS and \method$_{\text{Ada}}$, while having less parameters.

\begin{table*}[t]
\caption{Forecasting comparison between different extensions of \method~for handling multiple horizons and TTS method on \textsl{AirQuality} dataset.}
\label{Table_Multi_H_2}
\begin{center}
\begin{small}
\resizebox{.9\textwidth}{!}{
\begin{tabular}{rcccccccr}
\toprule
Method& MAE& MAPE & MSE& rMAE & rMSE & $|\Theta|$ & $t$ & Mem.\\ 

\midrule

& \multicolumn{8}{c}{$H=3$} \\  \cmidrule(l){2-9}

TTS&$14.697$&$0.327$&$\mathbf{803.452}$&$0.234$&$\mathbf{0.119}$&$19.2$ K&$13.50$&$3.03$\\

\method$_{\text{MLP}}$&$16.745$&$0.369$&$1090.712$&$0.267$&$0.162$& $15$&$\mathbf{0.76}$ & $\mathbf{0.05}$\\

\method$_{\text{Sha}}$&$\underline{14.620}$&$\underline{0.317}$&$810.962$&$\underline{0.233}$&$\underline{0.12}$&$\underline{12}$ & $\underline{1.48}$&$\underline{0.12}$\\

\method$_{\text{Ada}}$&$\mathbf{14.553}$&$\mathbf{0.315}$&$\underline{804.710}$&$\mathbf{0.232}$&$\mathbf{0.119}$ & $30$ & $1.49$ &$0.13$\\
\midrule

& \multicolumn{8}{c}{$H=6$} \\  \cmidrule(l){2-9}

TTS&$18.903$&$0.443$&$1142.317$&$\underline{0.301}$&$\underline{0.170}$&$19.4$ K&$13.29$&$3.04$\\

\method$_{\text{MLP}}$&$19.966$&$0.456$&$1306.661$&$0.318$&$0.194$&$\underline{18}$&$\mathbf{0.88}$&$\mathbf{0.06}$\\

\method$_{\text{Sha}}$&$\mathbf{18.736}$&$\mathbf{0.420}$&$\mathbf{1128.27}$&$\mathbf{0.299}$&$\mathbf{0.168}$&$\mathbf{12}$&$2.67$&$\underline{0.25}$\\

\method$_{\text{Ada}}$&$\underline{18.884}$&$\underline{0.434}$&$\underline{1132.528}$&$\underline{0.301}$&$\mathbf{0.168}$&$57$&$2.70$&$\underline{0.25}$\\

\midrule

 & \multicolumn{8}{c}{$H=9$} \\  \cmidrule(l){2-9}
TTS&$21.665$&$0.558$&$\underline{1344.903}$&$0.345$&$\underline{0.200}$&$19.5$ K& $13.10$&$3.05$\\
\method$_{\text{MLP}}$&$22.147$&$0.527$&$1469.130$&$0.353$&$0.218$&$\underline{21}$&$\mathbf{1.04}$&$\mathbf{0.07}$\\
\method$_{\text{Sha}}$&$\underline{21.501}$&$\underline{0.506}$&$1357.183$&$\underline{0.343}$&$0.201$&$\mathbf{12}$&$3.87$&$\underline{0.37}$\\
\method$_{\text{Ada}}$&$\mathbf{21.244}$&$\mathbf{0.504}$&$\mathbf{1336.466}$&$\mathbf{0.339}$&$\mathbf{0.198}$&$84$&$\underline{3.82}$&$\underline{0.37}$\\
\bottomrule
\end{tabular}
}
\end{small}
\end{center}
\end{table*}

\section{Synthetic Dataset}
\label{app:synthetic_dataset}

Our general aim here is to study the forecasting performance across different settings of limited spatiotemporal time series.
In this way, we first generate directed binary Erd\H{o}s-R\'enyi (ER) graphs with \(N\) nodes and edge probability \(p_{\text{\tiny ER}}\). 
To make the generation process stable, we force the existing edge weights to be uniformly distributed in the intervals of \([-0.3,-0.1]\) or \([0.1,0.3]\) as
\begin{equation}
\mathbf{A}_{ij} = bu_1+(1-b)u_2; \:\:\text{for}\:\:i\ne j= 1,\hdots,N,
\end{equation}
\noindent where
\begin{equation}
    b\sim \text{\textit{Bernoulli}}(1);\:u_1\sim\mathcal{U}(0.1,0.3);\:u_2\sim\mathcal{U}(-0.3,-0.1).
\end{equation}
The graph filter coefficients are also generated as \(c_{10}=0,\:c_{11}=1\) and 
\begin{equation}
\label{coeff_gen}
\begin{split}
&\theta_{ij}\sim\frac{\mathcal{U}(-1,-0.45)+\mathcal{U}(0.45,1)}{2^{i+j+1}}\\
&\text{for}\:\:2\le i\le M,\:0\le j\le i,
\end{split}
\end{equation}
\noindent to model the decreasing rate of the coefficients with distance from the current time step \cite{mei2016signal}. Next, after generating the first \(M\) time steps \(\{\mathbf{x}(i)\in\mathbb{R}^{N\times1}\}_{i=1}^M\) from the normal distribution, we generate \(\mathbf{x}_{k}\) for \(k > M\) using the proposed model (\ref{General}) with the graph filters (\ref{Def_GPi}) by considering \(\alpha_i=1\) and \(\sigma(.)=\tanh(.)\) for \(i=1,\hdots,M\). Studying the effect of noise measures is also of interest. Therefore, we generate the exogenous noise \(\mathbf{w}(k)\sim\mathcal{N}(\mathbf{0}_N,\mathbf{I}_N)\), where \(\mathbf{0}_N\in\mathbb{R}^{N\time1}\) and \(\mathbf{I}_N\in\mathbb{R}^{N\times N}\) denote the all-zero vector and Identity matrix, respectively. Then, the generated noise signal \(\mathbf{w}(k)\) is added to the resulting current time step signal with a manageable scalar \(\eta\) to control the desired amount of Signal-to-Noise Ratio (SNR) (in dB) as:
\begin{equation}
\label{CGP_GNN_SNR}
\begin{split}
\mathbf{x}_{k}&=\sum_{i=1}^M{\alpha_i\sigma\left(P(\mathbf{A},\boldsymbol{\theta}_i)\mathbf{x}(k-i)\right)} + \eta\mathbf{w}(k),\\
\end{split}
\end{equation}
\noindent where
\begin{equation}
\eta=10^{-\frac{SNR}{20}}\frac{\|\sum_{i=1}^M{\alpha_i\sigma\left(P(\mathbf{A},\boldsymbol{\theta}_i)\mathbf{x}(k-i)\right)}\|_2}{\|\mathbf{w}(k)\|_2}.
\end{equation}

\subsection{Synthetic experiments settings}
\label{syn_set}
The settings used for the generation of synthetic data can be found in Table \ref{Table_Settings} and also the provided implementation codes.

\begin{table*}[t]
\caption{Synthetic data settings.}
\label{Table_Settings}
\begin{center}
\begin{small}
\begin{tabular}{lcccccccccccccccr}
\toprule
 
\multicolumn{5}{c}{Varying SNR} & \multicolumn{5}{c}{Varying \(K\)}\\ [0.5ex] 

\cmidrule(l){1-5} \cmidrule(l){6-10} 

\(N\) & \(M\) & \(K\) & \(p_{\text{\tiny ER}}\) & \(n_e\)& \(N\) & \(M\) & \(SNR\) & \(p_{\text{\tiny ER}}\) & \(n_e\) \\
\cmidrule(l){1-5} \cmidrule(l){6-10}

$100$ & $3$ & $100$ & $0.03$ & $10000$ &

$100$ & $3$ & $0$ & $0.03$ & $10000$\\

\midrule
\multicolumn{5}{c}{Varying \(N\)} & \multicolumn{5}{c}{Varying \(M\)}\\
\cmidrule(l){1-5} \cmidrule(l){6-10} 

\(SNR\) & \(M\) & \(K\) & \(p_{\text{\tiny ER}}\) & \(n_e\) & \(N\) & \(M\) & \(SNR\) & \(p_{\text{\tiny ER}}\) & \(n_e\)\\ 
\cmidrule(l){1-5} \cmidrule(l){6-10} 

$-10$ & $3$ & $100$ & $0.03$ & $5000$ &

$1000$ & $3$ & $0$ & $0.03$ & $10000$ \\

\bottomrule
\end{tabular}
\end{small}
\end{center}
\vskip -0.1in
\end{table*}


\begin{table*}[t]
\caption{The forecasting results in terms of rMSE on the synthetic spatiotemporal time-series on the underlying SBM graphs across different settings.}
\label{Table_SBM}
\begin{center}
\begin{small}
\resizebox{\textwidth}{!}{
\begin{tabular}{rcccccccccccccccr}
\toprule
 
\multirow{3}{*}{Method} & \multicolumn{3}{c}{SNR} & \multicolumn{3}{c}{$K$} & \multicolumn{3}{c}{$N$} & \multicolumn{3}{c}{\(M\)} \\ [0.5ex] 
\cmidrule(l){2-4} \cmidrule(l){5-7} \cmidrule(l){8-10} \cmidrule(l){11-13}
 & $-10$ & $0$ & $10$ & $50$ & $100$ & $500$ & $100$ & $500$ & $1000$ & $3$ & $5$ & $7$  \\ 
\midrule

DCRNN & $1.010$ & $0.789$ & $0.356$

& $0.808$ & $0.784$ & $0.762$

& $0.764 $ & $0.848 $ & $0.892$

 & $0.893 $ & $0.890 $ & $0.893$\\ 

TTS & $1.561 $ & $1.115 $ & $0.495$ &

$1.797 $ & $1.642 $ & $0.771$

 & $0.985 $ & $1.117 $ & $0.994$ 

 & $1.252 $ & $1.570 $ & $3.966$ \\ 

GCLSTM & $1.032 $ & $0.800 $ & $0.357$

 & $0.822 $ & $0.792 $ & $0.763$

 & $0.772 $ & $0.849 $ & $0.892$

 & $0.894 $ & $0.894 $ & $0.900$\\ 






GConvGRU & $1.010$ & $0.791$ & $0.356$

& $0.809$ & $0.786$ & $0.762$

& $0.768$ & $0.849$ & $0.892$

& $0.893$ & $0.890$ & $0.894$\\ 

GConvLSTM & $1.037$ & $0.803$ & $0.357$ 

& $0.832$ & $0.796$ & $0.763$

& $0.772$ & $0.850$ & $0.892$

& $0.894$ & $0.894$ & $0.902$\\ 






\method & $\mathbf{0.949}$& $\mathbf{0.706}$ & $\mathbf{0.301}$

& $\mathbf{0.706}$ & $\mathbf{0.708}$ & $\mathbf{0.706}$

& $\mathbf{0.713}$ & $\mathbf{0.711}$ & $\mathbf{0.707}$

& $\mathbf{0.708}$ & $\mathbf{0.706}$ &  $\mathbf{0.707}$\\ 

\bottomrule
\end{tabular}
}
\end{small}
\end{center}
\vskip -0.1in
\end{table*}

\begin{table*}[t]
\caption{Forecasting comparison between \method~and previous methods in eight real-world datasets.}
\label{Table_all_metrics}
\begin{center}
\resizebox{.8\textwidth}{!}{
\begin{tabular}{rcccccccccccccccr}
\toprule
\multirow{2}{*}{Method} & \multicolumn{3}{c}{\textsl{AirQuality}} & \multicolumn{3}{c}{\textsl{LargeST}} \\  \cmidrule(l){2-4} \cmidrule(l){5-7}

& MAE& MAPE & MSE & MAE& MAPE & MSE\\ 

\midrule

TTS & $10.354$ & $0.224$ & $491.958$&
$14.849$&$128169.875$&$598.24$\\

DCRNN&$10.61$ & $0.221$ & $530.825$ &
$15.704$&$52718.066$&$660.086$\\

A3TGCN & $10.665$ & $0.226$ & $517.683$&
$16.609$&$65608.055$&$732.923$\\

GCLSTM & $10.596$&$0.221$&$532.774$&
$15.684$&$49706.496$&$657.092$\\

GConvGRU&$10.589$&$0.221$&$531.911$&
$15.704$&$57567.344$&$660.231$\\

GConvLSTM&$10.575$&$0.22$&$530.077$&
$15.698$&$52633.637$&$658.892$\\

TGCN&$10.582$&$0.22$&$530.473$&
$15.919$&$66013.200$&$679.385$\\

\method~&$10.624$&$0.222$&$511.092$&
$15.692$&$85322.266$&$657.387$\\

\midrule

 & \multicolumn{3}{c}{\textsl{MetrLA}} & \multicolumn{3}{c}{\textsl{PeMS04}} \\  \cmidrule(l){2-4} \cmidrule(l){5-7}

TTS&$2.156$&$0.051$&$14.844$&
$18.01$&$113534.69$&$847.225$\\

DCRNN&$2.21$&$0.053$&$16.829$&
$18.768$&$67772.36$&$892.054$\\

A3TGCN&$2.264$&$0.055$&$17.504$&
$18.826$&$65629.26$&$900.378$\\

GConvGRU&$2.207$&$0.053$&$16.843$&
$18.75$&$68195.54$&$891.649$\\

GConvLSTM&$2.207$&$0.053$&$16.847$&
$18.755$&$67050.484$&$891.018$\\

TGCN&$2.21$&$0.053$&$16.857$&
$18.752$&$71458.54$&$892.026$\\

\method~&$2.326$&$0.056$&$17.621$&
$18.888$&$67798.914$&$908.361$\\

\midrule
 & \multicolumn{3}{c}{\textsl{PeMS07}} & \multicolumn{3}{c}{\textsl{PeMS08}} \\  \cmidrule(l){2-4} \cmidrule(l){5-7}

TTS&$18.54$&$11695.263$&$850.821$&
$14.114$&$22370.670$&$468.200$\\

DCRNN&$19.36$&$8507.979$&$910.359$&
$14.542$&$21200.885$&$493.877$\\

A3TGCN&$19.601$&$7195.501$&$921.983$&
$14.589$&$20701.734$&$499.706$\\

GConvGRU&$19.308$&$7149.512$&$909.83$&
$14.519$&$20455.140$&$492.674$\\

GConvLSTM&$19.304$&$7157.646$&$909.89$&
$14.527$&$20657.875$&$493.266$\\

TGCN&$19.302$&$7987.125$&$913.163$&
$14.514$&$22077.957$&$495.304$\\

\method~&$19.59$&$12390.102$&$918.835$&
$14.602$&$21250.025$&$499.851$\\

\midrule
 & \multicolumn{3}{c}{\textsl{PemsBAY}} & \multicolumn{3}{c}{\textsl{PvUS}} \\  \cmidrule(l){2-4} \cmidrule(l){5-7}

TTS&$0.876$&$3637.020$&$2.905$&
$0.523$&$15305.67$&$4.927$\\

DCRNN&$0.91$&$3634.397$&$3.084$&
$0.539$&$5762.22$&$5.095$\\

A3TGCN&$0.976$&$3628.214$&$3.709$&
$0.73$&$6646.892$&$6.549$\\

GConvGRU&$0.909$&$3635.235$&$3.076$&
$0.548$&$6629.87$&$5.174$\\

GConvLSTM&$0.909$&$3634.943$&$3.068$&
$0.542$&$7945.95$&$5.1$\\

TGCN&$0.91$&$3634.395$&$3.082$&
$0.561$&$6935.281$&$5.49$\\

\method~&$0.982$&$3626.903$&$3.716$&
$0.642$&$17812.576$&$5.918$\\
\bottomrule
\end{tabular}
}
\end{center}
\end{table*}

\begin{table*}[t]
\caption{Comparison of Memory consumption (Mem), runtime ($t$), and number of learnable parameters ($|\Theta|$) between \method~and previous methods in six real-world datasets.}
\label{Table2_copy}
\begin{center}
\begin{small}
\begin{tabular}{rcccccccccccccccr}
\toprule
\multirow{2}{*}{Method} & \multicolumn{3}{c}{\textsl{MetrLA}} & \multicolumn{3}{c}{\textsl{AirQuality}} & \multicolumn{3}{c}{\textsl{PEMS-BAY}} \\  \cmidrule(l){2-4} \cmidrule(l){5-7} \cmidrule(l){8-10}

 & Mem & \(|\Theta|\) & \(t\) & Mem & \(|\Theta|\) & \(t\) & Mem & \(|\Theta|\) & \(t\)\\ 
\midrule

TTS & $5.61$ & $15.4$ K & $15.64$ &

$3.03$ & $19.1$ K
& $6.97$

& $6.70$ & $17.3$ K
& $36.04$

\\

DCRNN & $1.58$ & $6.8$ K & $7.21$ 

& $0.86$ & $6.8$ K
& $2.49$

& $1.89$ & $6.8$ K
& $18.4$
\\

A3TGCN & $6.53$ & $6.4$ K & $17.49$ 

& $3.53$ & $6.4$ K
& $7.25$

& $7.79$ & $6.4$ K
& $47.79$
\\



GCLSTM & $1.38$ & $4.7$ K & $9.40$

& $0.75$ &$ 4.7$ K
& $4.10$

& $1.65$ & $4.7$ K
& $27.48$
\\

GConvGRU & $1.21$ & $3.5$ K & $7.77$

& $0.66$ & $3.5$ K
& $2.62$

& $1.45$ & $3.5$ K
& $19.51$
\\

GConvLSTM & $1.55$ & $4.9$ K & $10.71$

& $0.84$ & $4.9$ K
& $4.55$

& $1.85$ & $4.9$ K
& $31.45$
\\

TGCN & $2.14$ & $6.6$ K & $7.43$

& $1.16$ & $6.6$ K
& $2.93$

& $1.16$ & $6.6$ K
& $20.29$
\\

\method & $\mathbf{0.08}$ & $\mathbf{12}$ & $\mathbf{4.48}$ 

& $\mathbf{0.05}$ & $\mathbf{12}$
& $\mathbf{0.92}$ 

& $\mathbf{0.09}$ & $\mathbf{12}$
& $\mathbf{8.07}$ 
\\

\midrule
\multirow{2}{*}{Method} & \multicolumn{3}{c}{\textsl{PeMS04}} & \multicolumn{3}{c}{\textsl{PeMS07}} & \multicolumn{3}{c}{\textsl{PeMS08}} \\ \cmidrule(l){2-4} \cmidrule(l){5-7} \cmidrule(l){8-10}

 & Mem & \(|\Theta|\) & \(t\) & Mem  & \(|\Theta|\) & \(t\) & Mem & \(|\Theta|\) & \(t\)\\ 

\midrule

TTS & $4.13$ & $17.0$ K &$7.61$ &

$4.93$ & $26.2$ K
& $38.18$

& $2.40$ & $14.8$ K
& $4.52$

\\ 

DCRNN & $1.17$ & $6.8$ K & $4.60$ 

& $1.39$ & $6.8$ K
& $22.55$

& $0.68$ & $6.8$ K
& $3.08$
\\

A3TGCN & $4.80$ & $6.4$ K & $12.84$ 

& $5.73$ & $6.4$ K
& $62.83$

& $2.80$ & $6.4$ K
& $7.95$
\\



GCLSTM & $1.02$ & $4.7$ K & $7.42$

& $1.21$ & $4.7$ K
& $35.64$

& $0.60$ & $4.7$ K
& $4.85$
\\

GConvGRU & $0.89$ & $3.5$ K & $4.83$

& $1.07$ & $3.5$ K
& $23.68$

& $0.51$ & $3.5$ K
& $2.94$
\\

GConvLSTM & $1.14$ &$ 4.9$ K & $8.07$

& $1.36$ & $4.9$ K
& $39.00$

& $0.67$ & $5.3$ K
& $4.17$
\\

TGCN & $1.57$ & $6.6$ K & $4.98$

& $1.88$ & $6.6$ K
& $24.25$

& $0.92$ & $6.6$ K
& $3.41$
\\

\method & $\mathbf{0.06}$ & $\mathbf{12}$ & $\mathbf{1.47}$

& $\mathbf{0.07}$ & $\mathbf{12}$
& $\mathbf{8.12}$

& $\mathbf{0.04}$ & $\mathbf{12}$
& $\mathbf{1.43}$  \\

\bottomrule
\end{tabular}
\end{small}
\end{center}
\vskip -0.1in
\end{table*}

\begin{table*}[t]
\caption{Comparison of Memory consumption (Mem), runtime ($t$), and number of learnable parameters ($|\Theta|$) between \method~and previous methods in two large-scale real-world datasets.}
\label{Table2_LargeST}
\begin{center}
\begin{small}
\begin{tabular}{rcccccccccccccccr}
\toprule
 
\multirow{2}{*}{Method} & \multicolumn{3}{c}{\textsl{LargeST}} & \multicolumn{3}{c}{\textsl{PvUS}} \\ [0.5ex] \cmidrule(l){2-4} \cmidrule(l){5-7} 

 & \(|\Theta|\) & \(t\) & Mem ($\beta=0.01$) & \(|\Theta|\) & \(t\) & Mem ($\beta=0.001$)\\ 

\midrule

TTS & $35.1$ K 
& $33.03$ & $11.90$

 & $20.1$ K & $31.72$ & $1.08$
\\ 

DCRNN & $6.9$ K 
& $26.54$ & $3.35$

& $6.9$ K & $33.85$ & $0.30$

\\

A3TGCN & $6.4$ K
& $64.08$ & $13.86$

& $6.5$ K & $75.60$ & $1.39$
\\



GCLSTM & $5.1$ K
& $34.21$ & $2.92$

& $4.8$ K & $40.10$ & $0.26$
\\

GConvGRU & $3.8$ K
& $22.93$ & $2.56$

& $3.6$ K & $28.79$ & $0.23$

\\ 

GConvLSTM & $5.0$ K 
& $37.62$ & $3.28$

& $5.0$ K & $43.92$ & $0.29$

\\

TGCN & $6.7$ K & $25.49$ &$4.53$

& $6.7$ K & $28.16$ & $0.46$

\\ 

\method & $\mathbf{12}$ & $\mathbf{9.02}$ & $\mathbf{0.43}$

& $\mathbf{12}$ & $\mathbf{12.75}$ & $\mathbf{0.13}$
 
\\ 

\bottomrule
\end{tabular}
\end{small}
\end{center}
\vskip -0.1in
\end{table*}

\begin{table*}[t]
\caption{Forecasting comparison between \method~and previous methods on AirQuality dataset when $M>3$.}
\label{Table_M6}
\begin{center}
\resizebox{\textwidth}{!}{
\begin{tabular}{rcccccccccccccccr}
\toprule
\multirow{2}{*}{Method} & \multicolumn{4}{c}{\textsl{$M=6$}} & \multicolumn{4}{c}{\textsl{$M=9$}} \\  \cmidrule(l){2-5} \cmidrule(l){6-9}

& MSE& \(|\Theta|\) & \(t\) & Mem. (\(\beta=0.2\)) & MSE & \(|\Theta|\) & \(t\) & Mem. (\(\beta=0.01\))\\ \midrule

TTS & $\mathbf{492.015}$ & $15.4$ K & $95.90$ & $5.36$&

$\mathbf{488.279}$& $19.1$ K  & $143.85$ & $8.04$\\

DCRNN & $530.833$ & $7.4$ K & $28.40$ & $0.88$

&  $532.249$ & $8.0$ K & $42.60$& $1.33$\\

A3TGCN & $515.126$ & $6.5$ K & $136.10$ & $6.94$

&  $511.945$ & $6.5$ K & $204.15$ & $10.41$\\

GCLSTM & $525.902$ & $5.2$ K & $42.50$& $0.76$

& $523.658$ & $5.5$ K & $63.75$ & $1.14$\\

GConvGRU & $522.120$ & $3.9$ K & $27.00$& $0.67$

& $524.201$ & $4.2$ K & $40.50$ & $1.10$\\

GConvLSTM &  $527.754$ & $5.4$ K & $45.30$& $0.85$

& $521.450$ & $5.8$ K & $67.95$ & $1.28$\\

TGCN & $524.758$ & $6.9$ K & $29.80$ &  $1.17$

&  $524.758$ & $7.2$ K & $44.70$ & $1.80$\\ 

\method~(ours) & $\underline{498.229}$ & $\mathbf{33}$ & $\mathbf{15.60}$ & $\mathbf{0.11}$

& $\underline{499.127}$ & $\mathbf{63}$ & $\mathbf{23.40}$ & $\mathbf{0.17}$ \\

\bottomrule
\end{tabular}
}
\end{center}
\end{table*}

\begin{table*}[t]
\caption{Runtime (in minutes) and memory usage (in GB) comparison on the LargeST dataset by running on CPU for only one epoch.}
\label{Table_CPU}
\begin{center}
\resizebox{.9\textwidth}{!}{
\begin{tabular}{rcccccccccccccccr}
\toprule

&TTS & DCRNN & A3TGCN & GCLSTM & GConvGRU & GConvLSTM & TGCN & \method~(ours)\\ \midrule

Runtime & $4.06$ & $2.66$ & $14.84$ & $\underline{0.95}$ & $1.15$& $1.87$ & $5.03$  & $\mathbf{0.22}$\\ \midrule

Mem. & $11.90$ & $3.35$ & $13.86$ & $2.92$ & $\underline{2.56}$ & $3.28$ & $4.53$ & $\mathbf{0.43}$\\

\bottomrule
\end{tabular}
}
\end{center}
\end{table*}

\begin{table*}[t]
\caption{MSE results of ``Avg'' and ``Last'' baselines compared to \method.}
\label{Table_SimpleBaseline}
\begin{center}
\begin{tabular}{rcccccccccccccccr}
\toprule

Method &MetrLA & LargeST & PeMS04 & PeMS07 & PeMS08 & PvUS\\ \midrule

Avg & $20.645$ & $1125.028$ & $933.061$ & $991.629$ & $529.050$& $11.881$\\ 

Last & $18.803$ & $712.786$ & $1169.624$ & $1082.724$ & $603.319$ & $6.141$\\

\method & $17.621$ & $657.387$ & $908.361$ & $918.835$ & $499.851$ & $5.918$\\

\bottomrule
\end{tabular}
\end{center}
\end{table*}

\begin{table*}[t]
\caption{Mean and standard deviation of stability analysis.}
\label{STD_stab}
\begin{center}
\begin{tabular}{rcccccccccccccccr}
\toprule

&$p=0.1$&$p=0.3$&$p=0.5$&$p=0.7$\\ \midrule

SNR=15&0.076$\pm$0.003&0.105$\pm$0.005	&0.111$\pm$0.006&0.114$\pm$0.005\\ 

SNR=0&0.118$\pm$0.014&0.140$\pm$0.021	&0.173$\pm$0.014&0.212$\pm$0.014\\

SNR=-15&0.121$\pm$0.019&0.158$\pm$0.023&	0.249$\pm$0.140&0.264$\pm$0.077\\
\bottomrule
\end{tabular}
\end{center}
\end{table*}

\begin{table*}[t]
\caption{Standard deviation of MSE metrics.}
\label{STD_stab2}
\begin{center}
\begin{tabular}{rcccccccccccccccr}
\toprule

&AirQuality&LargeST&PeMS08&PemsBay\\\midrule

TTS&0.597&2.226&0.765&0.243\\

DCRNN&1.477&1.925&0.192&0.167\\

A3TGCN&0.996&0.0495&0.543&0.124\\

GCLSTM&0.908&0.783&0.403&0.119\\

GConvGRU&0.48&1.011&0.366&0.267\\

GConvLSTM&1.316&0.807&0.080&0.382\\

TGCN&0.316&11.573&0.395&0.485\\ 

\method &0.892&0.604&0.394&0.192\\

\bottomrule
\end{tabular}
\end{center}
\end{table*}

\subsection{Synthetic data with Stochastic Block Model (SBM) graphs}
Here, to show the flexibility of the proposed framework on other graph structures, we considered the directed Stochastic Block Model (SBM) with three communities. Besides, the community input and output edge probabilities are set as \(p_{\text{in}}=0.3\) and \(p_{\text{out}}=0.01\), respectively. We repeated the experiments on the synthetical time series on the underlying SBM graphs and put the results in Table \ref{Table_SBM}. As can be illustrated in this table, the superior and more robust reconstruction performance of the proposed \method~ framework against the SOTA in the case of limited data can be observed, which shows that the proposed framework can flexibly handle and benefit from different graph structures in the forecasting task. 

\section{Results on real-world datasets}
\label{AddReal}

\subsection{Hyperparameters}
\label{real_set}
In experiments on any of real-world datasets, we used the learning rate \(lr=0.01\), the ADAM optimizer, and the number of epochs \(n_{e}=1000\) (except for the LargeST which was \(n_e=100\)).
\subsection{Additional datasets}
\begin{itemize}
\item \textsl{MetrLA} \cite{li2018diffusion}: 
Consists of traffic reading metrics recorded on 207 highway loop detectors in Los Angeles County. These metrics were then aggregated in 5-minute segments over a four-month interval from March 2012 to June 2012.



\item \textsl{PeMS04} \cite{guo2021learning}: 
5-minute traffic readings metrics for a 2-month interval from 01/01/2018 to 02/28/2018 recorded by 307 traffic sensors in the San Francisco Bay Area.

\item \textsl{PeMS07} \cite{guo2021learning}: 
5-minute traffic readings metrics for a 4-month interval from 05/01/2017 to 08/31/2017 recorded by 883 traffic sensors in the San Francisco Bay Area.

\item \textsl{PvUS} \cite{Cini_Torch_Spatiotemporal_2022}: 
The quantified production of simulated solar power recorded by about 5,000 photovoltaic plants across the United States provided by \href{https://www.nrel.gov/}{ National Renewable Energy Laboratory (NREL)’s Solar Power Data for Integration Studies}.
\end{itemize}

In addition to the rMSE metric, we also consider Mean absolute error (MAE), Mean absolute percentage error (MAPE), MSE, and relative MAE (rMAE). The results over all of the introduced real-world datasets in terms of these metrics have been provided in Table \ref{Table_all_metrics}. Besides, the comparison of \(|\Theta|\), \(t\), and also memory usage over these datasets have been provided in Tables \ref{Table2_copy}-\ref{Table2_LargeST}. As can be deducted from these comprehensive results, the stated interpretations from the results of the main paper body are more emphasized relying that the proposed \method~significantly outperforms previous methods in terms of memory consumption and running times, all without compromising forecasting accuracy.
These results make our model highly deployable in real-world scenarios with resource constraints, promising more efficient and accessible solutions for a wide range of applications in forecasting.

\subsection{Additional results when $M$ exceeds 3}

Table \ref{Table_M6} shows additional results in the AirQuality dataset when $M>3$.
We observe that the results remain stable with bigger values of $M$.

\subsection{Runtime comparison (in minutes) on LargeST dataset by running on CPU for only one epoch}

In Table \ref{Table_CPU}, we train all the involved methods for one epoch on the large-scale dataset LargeST using only CPUs to simulate cases for possible lack of access to GPU resources.
Considering we need at least $100$ epochs to get decent forecasting performance, the nearest method to the ideal case of real-time processing is the proposed \method.


\subsection{Simplest baselines (mean of samples \& last sample)}

We have added two trivial baselines for comparison.
We consider "Avg" and "Last" as two scenarios of taking the average of the window samples and the last time sample as the forecasting prediction and have evaluated the results in Table \ref{Table_SimpleBaseline}.

\section{Standard deviations of the main results}

The standard deviation of the stability analysis in Figure \ref{fig:LimitedData} and forecasting performance in Table \ref{Table2} are provided in Table \ref{STD_stab} and \ref{STD_stab2}, respectively.

\section{Experiments Compute Resources}
All the experiments were run on one GTX A100 GPU device with 40 GB of RAM.

\end{document}